\documentclass[conference,a4paper]{IEEEtran}
\IEEEoverridecommandlockouts
\usepackage[left=1.59cm, right=1.59cm, top=1.9cm, bottom=4.4cm]{geometry}

\usepackage{cite}
\usepackage{amsmath,amssymb,amsfonts}
\usepackage{algorithmic}
\usepackage{graphicx}
\usepackage{mathtools}
\usepackage{textcomp}
\usepackage{xcolor}
\usepackage[most]{tcolorbox}
\usepackage{amsthm}

\makeatletter
\let\MYcaption\@makecaption
\makeatother

\usepackage[font=footnotesize]{subcaption}

\makeatletter
\let\@makecaption\MYcaption
\makeatother

\usepackage{booktabs}
\usepackage{multirow}

\usepackage{xspace}
\newcommand{\ie}{\emph{i.e., }}
\newcommand{\eg}{\emph{e.g., }}
\newcommand{\dataset}[1]{\textsc{#1}}
\newtheorem{proposition}{Proposition}
\newtheorem*{proposition*}{Proposition}
\newtheorem{assumption}{Assumption}
\newtheorem{definition}{Definition}
\newtheorem{lemma}{Lemma}
\newtheorem{corollary}{Corollary}

\DeclareMathOperator*{\argmin}{arg\,min}

\newcommand{\txtmode}[1]{\textsc{\textsf{#1}}}

\def\BibTeX{{\rm B\kern-.05em{\sc i\kern-.025em b}\kern-.08em
    T\kern-.1667em\lower.7ex\hbox{E}\kern-.125emX}}
\begin{document}

\title{Compression Through Time: An Interplay of Separability, Stationarity and Capacity\\
\thanks{This research was funded in part by the Austrian Science Fund (FWF) under grant number 10.55776/PAT7753623. }
}

\author{\IEEEauthorblockN{Nikolaus Kopp}
\IEEEauthorblockA{%
\textit{Graz University of Technology}\\
Graz, Austria\\
nkopp@tugraz.at}
\and
\IEEEauthorblockN{Franz Pernkopf}
\IEEEauthorblockA{%
\textit{Graz University of Technology}\\
Graz, Austria\\
pernkopf@tugraz.at}
\and
    \IEEEauthorblockN{Bernhard C. Geiger}
    \IEEEauthorblockA{
        \textit{Graz University of Technology} \&\\
        \textit{Know Center Research GmbH}\\
        Graz, Austria\\
        geiger@tugraz.at}
}

\maketitle

\begin{abstract}
Information-theoretic and geometric compression have been widely studied as possible explanations for generalization. Yet their temporal dynamics in recurrent networks remain largely unexplored. Using discretized LSTMs, whose finite embedding space enables direct estimates of information quantities, we observe pronounced compression across the final timesteps of many datasets. We explain this effect through three competing mechanisms: (i) class separability permits stronger compression at the final timestep; (ii) shared recurrent dynamics propagate this pressure across timesteps with similar activation states, making its extent depend on temporal stationarity; and (iii) limited representational capacity opposes prolonged compression because it restricts information available for subsequent computation. We systematically test these mechanisms through controlled synthetic interventions and real-world sequence tasks. Our results identify class separability, temporal stationarity, and network capacity as key determinants of temporal compression, and suggest that compression acts as a regularizer---potentially beneficial when excess capacity is available, but harmful to generalization when it consumes memory needed for the task.

\end{abstract}

\section{Introduction}
Neural networks (NNs) owe much of their success to empirical advances and computational scale, yet our theoretical understanding of their inner workings remains limited. Information theory has emerged as a promising lens for analyzing NN dynamics, particularly following the information plane analysis of Shwartz-Ziv and Tishby~\cite{shwartz2017opening}. By tracking the information content of layer activations during training, this work popularized the idea that internal embeddings may undergo distinct expansion and compression phases. A parallel line of work revealed striking geometric regularities in learned representations. Most notably, neural collapse shows that, late in training, penultimate-layer embeddings can concentrate to their class means, which approach a highly symmetric configuration~\cite{papyan2020prevalence}. 

These lines of work have sparked broad interest in compression and representation geometry not only as correlates of generalization, but as possible mechanisms underlying it. In this view, the progressive contraction of representations helps discard task-irrelevant variation, improve class separation, and thereby promote robustness and generalization. A substantial body of work has argued both for and against this view, leaving the question open~\cite{saxe2019information, geiger2021information, lorenzen2022information, wu2024linguistic, gao2023towards}.

Information-theoretic analyses have also been applied to recurrent NNs, \eg to study information processing and task computation~\cite{hintze2025identifying,tomoda2026emergence}, predictive information~\cite{dong2019predictive}, and the relationship between network capacity and trainability~\cite{collins2017capacity}. Similar investigations have also been conductedfor spiking NNs~\cite{engelken2022time,tomoda2026emergence}. Only a few studies have considered temporal information dynamics in recurrent systems: Engelken et al.~\cite{engelken2022time} analyze such dynamics in spiking NNs, while Baker et al.~\cite{baker2022information} record temporal trajectories in conventional RNNs without analyzing the mechanisms that shape them. This is a notable gap, since recurrence enables networks to manage information over time, making recurrent networks a particularly interesting setting for studying information dynamics. This motivates our study of temporal information dynamics in recurrent networks.

To make this analysis tractable, we consider LSTMs with binary hidden state embeddings (Section~\ref{sec:methodology}). The finite state space permits the reliable estimation of information-theoretic quantities, avoiding challenges associated with continuous representations. This controlled setting reveals a phenomenon superficially similar to the expansion and compression phases observed in information plane analyses: as shown in Fig.~\ref{fig:intro_figure}, on some datasets the entropy of the final recurrent layer embedding initially expands, then compresses as more elements of the sequence are processed. Since these phases evolve over sequence indices (\ie time) rather than epochs, since recurrent NNs behave inherently different from feed-forward NNs, and since estimation artifacts are carefully ruled out as confounders explaining these phases, this observed behavior calls for explanation. To this end, we formulate a series of hypotheses (Section~\ref{sec:hypothesis}) that disentangle the effects of several factors on entropy. Tested on synthetic and real-world datasets (Section~\ref{sec:setup}), we experimentally analyze the mechanisms shaping temporal compression in Section~\ref{sec:experiments}.

\begin{figure}[t]
    \centering
    \includegraphics[height=3.8cm]{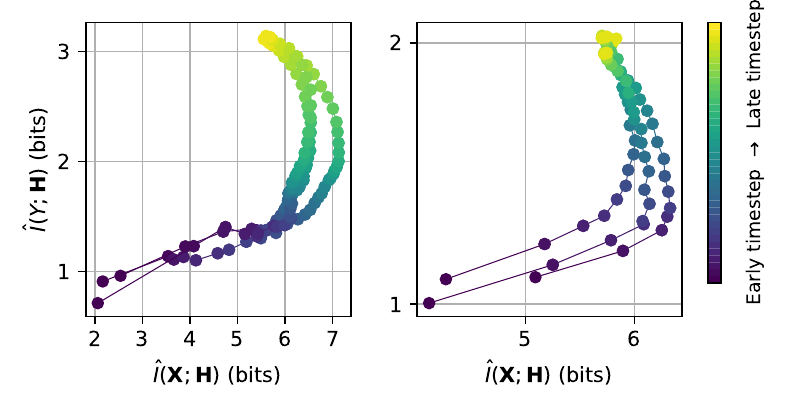}

    \caption{Embedding entropy and target information of last recurrent layer of an LSTM as it evolves over time. Each trajectory corresponds to a training trial. \textbf{Left: } Dataset \dataset{MNIST stroke}. \textbf{Right:} Dataset \dataset{Tiselac}}
    \label{fig:intro_figure}
\end{figure}

\section{Methodology}
\label{sec:methodology}
We consider classification tasks, where each input is a sequence $\mathbf{x} = \{\mathbf{x}^t\}_{t=1}^T$ of $T$ elements $\mathbf{x}^t \in \mathbb{R}^F$. Sequences may vary in length $T$, but the feature dimension $F$ is fixed for a task. Each sequence is associated with one of $C$ classes via class label $y \in \{1,\ldots,C\}$. We assume that sequences and labels are jointly distributed according to $p_{\mathbf{X}, Y}$. Given a dataset $\mathcal{D} = \{(\mathbf{x}_{(n)}, y_{(n)})\}_{n=1}^N$ of $N$ samples of the unknown joint $p_{\mathbf{X}, Y}$, the goal is to learn a function $\hat{f}_{\mathcal{W}}{:}\ \bigcup_{T\in\mathbb{N}} \mathbb{R}^{T\times F} \to \Delta^{C-1}$, parametrized by weights $\mathcal{W}$, which maps sequences of arbitrary length to the probability simplex $\Delta^{C-1}$. The parameters $\mathcal{W}$ are optimized by minimizing an empirical loss over the training set $\mathcal{D}_{tr} \subset \mathcal{D}$. We refer to sequence index $t$ as timestep and to the associated dimension as \emph{time}. Further, we will use upper-case letters for random variables (RVs) such as $\mathbf{X}$ and $Y$, and lower-case letters for their realizations. 

Classification is performed via recurrent neural structures.
They consist of at least one recurrent layer that sequentially processes elements $\mathbf{x}^t$ while keeping a state $\mathbf{s}^{t-1}$ that carries information from previously processed elements.
We use LSTMs, whose computations are as follows: The concatenation $\mathbf{z}^t = [\mathbf{x}^t; \mathbf{h}^{t-1}]$ is first passed through four perceptrons $\phi_\kappa(\mathbf{W}_\kappa \mathbf{z}^t + \mathbf{b}_\kappa) = \boldsymbol{\kappa}^t \in \mathbb{R}^{d}$, for $\kappa \in \{i,f,o,c\}$, where $d$ denotes the number of units in the respective layer and $\phi_\kappa \in \{\operatorname{tanh}, \operatorname{sigmoid}\}$ is a component-wise activation function. The computed quantities are then used to update the cell state $\mathbf{s}^t = \mathbf{f}^t \odot \mathbf{s}^{t-1} + \mathbf{i}^t \odot \mathbf{c}^t$ and the (pre-)embedding $\tilde{\mathbf{h}}^t = \mathbf{o}^t \odot \tanh(\mathbf{s}^t)$. While in conventional LSTMs $\mathbf{h}^t = \tilde{\mathbf{h}}^t$, we apply layer normalization $\operatorname{LN}(\cdot)$ followed by a component-wise sign quantizer $\mathcal{Q}{:}\ \mathbf{R}^{d} \longrightarrow \{-1, 1\}^{d}$ to obtain the embedding, \ie $\mathbf{h}^t = \mathcal{Q}\left(\operatorname{LN}(\tilde{\mathbf{h}}^t)\right)$. The weights $\mathbf{W}_\kappa$ and biases $\mathbf{b}_\kappa$ make up the trainable parameters $\mathcal{W}$ of the recurrent layer.
Both cell state $\mathbf{s}^t$ and embedding $\mathbf{h}^t$ make up the information in the recurrent path but only $\mathbf{h}^t$ is sent to the next layer. 




We are interested in analyzing the information content of the binary embedding RV $\mathbf{H}^t$, specifically we want to track how mutual information $I(\mathbf{X}, \mathbf{H}^t)$ and $I(\mathbf{H}^t, Y)$ change over time $t$. Since $\mathbf{H}^t$ is obtained from a deterministic transformation of $\mathbf{X}$, we have $I(\mathbf{X}, \mathbf{H}^t) = H(\mathbf{H}^t)$ and therefore refer to $I(\mathbf{X}, \mathbf{H}^t)$ as \emph{embedding entropy}. Since $\mathbf{H}^t$ has $2^d$ possible states, for small $d$ its entropy can be estimated reliably from the relative frequencies of $\mathbf{h}^t$ observed on a test set $\mathcal{D}_{te}$ (plug-in estimate). Similarly, we approximate the \emph{target information} $I(\mathbf{H}^t, Y)$ by first estimating the joint distribution $p_{\mathbf{H}^t,Y}$ from the relative frequencies of $(\mathbf{h}^t,y)$ pairs and then substituting it in the mutual information formula.

\section{Hypotheses}
\label{sec:hypothesis}
In this section, we discuss and hypothesize potential mechanisms underlying the observed compression effect. First, we use established insights on neural collapse to explain compression at the final timestep. Second, we propose a mechanism through which objectives acting at some timesteps, \eg compression, can propagate to other timesteps. Third, we discuss an opposing force, namely reduced expressiveness.

\textbf{A) Last timestep compression.}

Neural collapse~\cite{papyan2020prevalence} states that, for separable classes, the cross-entropy loss combined with a softmax classifier tries to maximize the distance between penultimate-layer embeddings belonging to different classes. We argue, that a similar effect occurs in discrete embedding space. Here, this amounts to moving embeddings to the vertices of this hypercube that are maximally far apart. Indeed, in Appendix~\ref{sec:theory}, we show that, for a frozen softmax classifier, standard training of binary-activation NNs locally moves every pre-embedding $\tilde{\mathbf h}^T$ towards their unique loss-minimizing orthant. In essence, information compression can be considered a discrete-valued counterpart of neural collapse~\cite{papyan2020prevalence}.

\textbf{B) Constraint propagation.}
Having a hypothesis for compression at the last timestep, the observed compression already occurring during the preceding few timesteps demands an explanation. We posit that the same effect is responsible, and that the spill-over to other timesteps is caused by the mechanics of the recurrent structure. Specifically, we conjecture that an objective imposed on an LSTM at one timestep generalizes to other timesteps at which the LSTM occupies a similar state, \ie when its activation pattern $\mathbf{z}^t$ and, to some extent, cell state $\mathbf{s}^t$ are similar. This follows because these quantities are the only ones that distinguish one timestep from another and therefore also the only ones to carry temporal information. Consequently, the operations performed by an LSTM can differ across timesteps either because the input is non-stationary or because the LSTM uses some of its computational capacity to encode temporal information. If neither effect is strong, any objective acting at one timestep will similarly affect other timesteps. Indeed, activation distributions have been observed to approach stationarity~\cite{cooijmans2017recurrent}. We argue that this, combined with last timestep compression, explains the compression observed in Fig.~\ref{fig:intro_figure}.

\textbf{C) Compression hurts expressiveness.}
While the preceding discussion focused on mechanisms that promote compression across timesteps, there is also an opposing effect: reduced network expressiveness. As embeddings become more compressed, the space of functions representable by the LSTM is increasingly restricted. 
This argument, however, is not limited to discrete representations. Even though continuous representations can, in principle, encode unbounded amounts of information, objectives such as $L_2$ regularization are well known to reduce network complexity.

\textbf{Combined effect.}
In summary, we speculate that the softmax objective creates competing forces on embedding entropy. On the one hand, the loss can benefit from compression at the final timestep; on the other hand, if this compression propagates through time, it can impair the network's ability to make accurate predictions. Based on the mechanisms described above, we hypothesize that the extent of compression is largely determined by three factors: (i) The separability of the data determines the compression at the final timestep. (ii) The stationarity of the activation patterns determines how compression propagates across timesteps. (iii) The excess representational capacity of the network determines how much capacity can be devoted to compression or to a circuit that tracks time and initiates compression at late timesteps.

\section{Experimental Setup}
\label{sec:setup}
\paragraph{Datasets}
We use the processed real-world datasets \dataset{MNIST digit stroke}\cite{dejong2016mniststrokes}, \dataset{Tiselac}\cite{tiselac}, and \dataset{ElectricDevices}~\cite{electricdevices}. For \dataset{MNIST digit stroke}, we discard digits drawn with more than three strokes to remove outliers, trim end-of-sequence tokens, and resample all strokes to a common sequence length of $T=38$. 
%
%
We furthermore generate synthetic sequences--label pairs $\{(\mathbf{x}_n, y_n)\}_n$ based on class-specific AR(3) processes with reflection coefficients $\mathbf{k}^y=[0.7,\,0.7,\,0.5(-1)^y]$, where $y\in\{0,1\}$ and set the innovation variance such that each process has unit marginal variance. For details on dataset creation and preprocessing see Appendix~\ref{app:datasets}.

\paragraph{End-of-sequence signal}
As an intervention, we add an end-of-sequence (EOS) signal to the last recurrent layer in some experiments. In these cases, we expand the layer's input with an additional feature dimension that is set to 0 for all but the final $t_\mathrm{EOS}$ timesteps.

\paragraph{Network and training setup}
We use 1-2 LSTM layers followed by a fully connected feed-forward layer with softmax activation. The loss is the cross-entropy between the network predictions and the one-hot-encoded labels. Weights are updated using backpropagation through time. Since the quantizer $\mathcal{Q}(\cdot)$ is non-differentiable, its derivative is replaced by a \emph{surrogate}, namely the derivative of $\operatorname{tanh}$. This heuristic is known as the \emph{straight-through estimator} (STE)~\cite{bengio2013estimating}.

We report results using \emph{post-training early stopping}, \ie networks are trained for a fixed number of epochs ensuring convergence, while performance on the validation and test sets is evaluated throughout training. After training, we identify the epoch at which validation accuracy is highest and report the corresponding test-set performance. This ensures a consistent and fair evaluation, as binarized recurrent architectures exhibit substantial training instability~\cite{kopp2026learning}. 

\section{Experiments}
\label{sec:experiments}
We structure the experiments around the three components of our hypothesis. We first isolate the mechanisms that promote compression. Then we examine how limited temporal information affects the propagation of compression, and, finally, we investigate the relationship between compression, network capacity, and generalization.

\subsection{Last timestep compression}


Complementing our theoretical results in Appendix~\ref{sec:theory}, we first experimentally verify that (i) softmax induces  the penultimate embedding $\mathbf{H}^t$ compression at last timestep $T$, and (ii) that this compression effect depends on separability of pre-embeddings $\tilde{\mathbf{H}}^T$. To isolate the behavior of the loss on final timestep pre-embeddings, we directly generate and manipulate final timestep pre-embeddings $\tilde{\mathbf{h}}^T$. For $d=10$, we generate balanced $C$-class embeddings in $\mathbb{R}^d$ by sampling class means $\{\boldsymbol{\mu}_y\}_{y=1}^C$ uniformly from the unit $d$-ball and drawing $\tilde{\mathbf{h}}_n^T\sim\mathcal{N}(\boldsymbol{\mu}_{y},\mathbf{I}_d)$ for each label $y$, where $\mathbf{I}_d$ denotes the identity matrix. 

We pass these synthetic pre-embeddings through the quantizer $\mathcal{Q}(\cdot)$ followed by a softmax layer. We not only update the softmax layer with gradients, but also the pre-embeddings themselves. This setup can be understood as simulating an arbitrarily powerful encoder $\hat{f}_e$, where $\hat{f}_e(\mathbf{x}) = \tilde{\mathbf{h}}^T$. To model more realistic encoders, we add a penalty to the cross-entropy loss for separating pre-embeddings which were initially nearby: Denoting the initially generated and the adapted pre-embeddings by $\tilde{\mathbf{h}}^{(0)}$ and $\tilde{\mathbf{h}}$, the penalty is
\begin{equation*}
\mathcal{L}_{\mathrm{sep}} \! = \! \frac{\lambda_{sep}}{|\mathcal{P}|} \! \sum_{(i,j)\in\mathcal{P}} \! \! \! \operatorname{ReLU} \!  \left( \left\lVert \tilde{\mathbf{h}}_i-\tilde{\mathbf{h}}_j\right\rVert_2 \! - \left\lVert \tilde{\mathbf{h}}_i^{(0)}-\tilde{\mathbf{h}}_j^{(0)}\right\rVert_2\right),
\end{equation*}
where $\lambda_{sep}$ is a \emph{separation cost}, $\operatorname{ReLU}$ is the rectified linear function, and $\mathcal{P} = \left\{(i,j) \mid i \in [1,\dots, N],\; j \in \mathcal{N}_5(i)\right\}$, where $\mathcal{N}_5(\cdot)$ is the set of 5-nearest neighbors evaluated on pre-embeddings prior to training $\tilde{\mathbf{h}}^{(0)}$. We vary $\lambda_{sep}$ and optimize pre-embeddings for 5000 steps with a learning rate of $0.03$. Each experiment is repeated five times with different pre-embedding and weight initializations. Fig.~\ref{fig:simulated} shows the result. As expected, an arbitrarily powerful encoder, \ie $\lambda_{sep}=0$, learns to compress pre-embeddings of each class into a single quantization cell. A positive separation cost $\lambda_{sep}$, however, leads to embedding with higher entropy. Note that this effect happens even without a deterioation in classification performance.

\begin{figure}[t]
    \centering

    \includegraphics[width=0.65\linewidth]{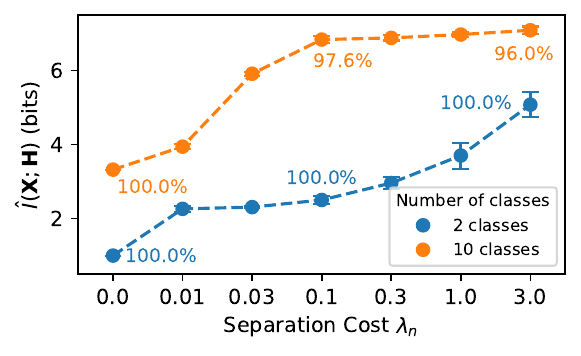}

    \caption{Empirically estimated embedding entropy $\hat{I}(\mathbf{X}; \mathbf{H})$ after directly optimizing pre-embeddings to minimize softmax loss. Entropy is plotted over the \emph{separation cost} $\lambda_{sep}$, which penalizes separating pre-embeddings which are initially close. The resulting classifier accuracies are shown for $\lambda_{sep} \in \{0, 0.1, 3.0\}$.}
    \label{fig:simulated}
\end{figure}

\subsection{Constraint Propagation}
\paragraph{EOS intervention}
Having theoretically and experimentally established that softmax with cross-entropy loss incentivizes low embedding entropy, we now investigate this effect in LSTMs with a single recurrent layer, trained on the synthetic AR dataset with sequence length $T=150$. In particular, we examine how compression propagates to earlier timesteps, and how an EOS signal, inducing temporal information and violating stationarity, affects propagation.

We consider three EOS signals $t_\mathrm{EOS}\in \{0,1,10\}$ ($t_\mathrm{EOS}=0$ means no EOS signal is present), and perform both training and evaluation using the same $t_\mathrm{EOS}$. Each experiment is repeated for five trials with independent weight initializations. Fig.~\ref{fig:ar_eos} shows the embedding entropy on the test set at the epoch of highest validation accuracy, and Table~\ref{tab:ar_eos} shows a summary of interesting datapoints. In all settings it can be seen that $\hat{I}(\mathbf{X}; \mathbf{H}^t)$ underwent a phase of substantial reduction, with the smallest values assumed at the last timestep $T$. Compression apparently propagates from this last timestep to previous timesteps as well.

Strikingly, informing LSTMs about EOS causes most of the embedding compression to be contained within the timesteps in which the EOS signal is active. This temporal localization manifests in two ways: (i) embedding entropy is substantially higher before the EOS signal, and (ii) it can reach lower values at the final timestep. A possible explanation follows from the dual role of the binary embedding: providing the penultimate embedding for readout and serving as context for LSTM state updates. Although the task is learned in all settings (see accuracies in Fig.~\ref{fig:ar_eos}), compression appears to have a negative impact on network flexibility. Since the EOS signal carries no information about the target, its associated weights can specialize in a single role—--invoking compression at the final timestep. This specialization may be easier to learn because it is consistent across samples. As the EOS-associated weights increasingly absorb the compression pressure, less compression is imposed on earlier timesteps. This frees up earlier embeddings, increasing network flexibility and potentially enabling faster convergence. As shown in Table~\ref{tab:ar_eos}, training converges slower in the absence of an EOS signal, suggesting that extended compression interferes with learning.


\begin{figure}[t]
    \centering
    \includegraphics[width=0.99\linewidth]{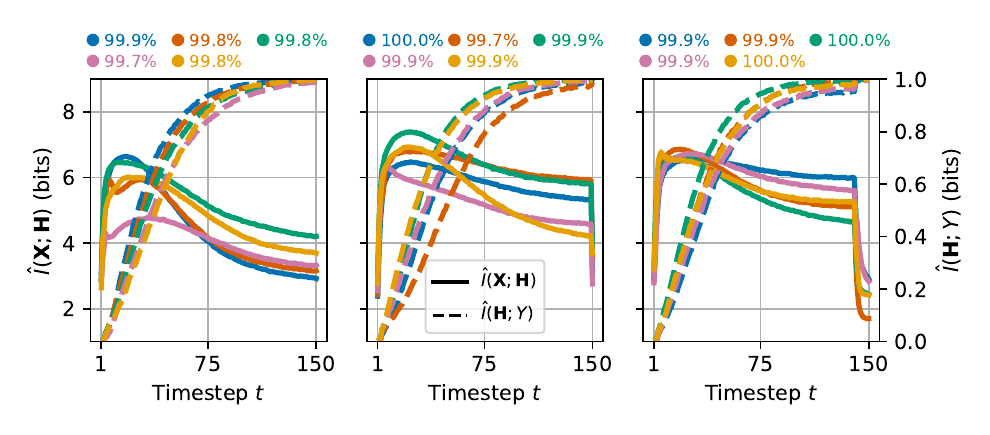}
    \caption{Embedding entropy $\hat{I}(\mathbf{X}; \mathbf{H}^t)$ and target information $\hat{I}(\mathbf{H}^t;Y)$ over timesteps $t$. LSTM networks are trained and evaluated on the synthetic AR dataset. Each experiment comprises five trials, shown in different colors. Test set accuracies are displayed above the subplots. \textbf{Left: } No EOS token ($t_\mathrm{EOS}=0$). \textbf{Middle: } EOS token only at final timestep ($t_\mathrm{EOS}=1$). \textbf{Right: } EOS token for final 10 timesteps ($t_\mathrm{EOS}=10$).}
    \label{fig:ar_eos}
\end{figure}

\begin{table}[t]
\centering
\caption{Embedding entropy $\hat{I}(\mathbf{X};\mathbf{H}^t)$ at $t \in [140, 150]$ and their difference. The epoch at which training accuracy surpassed $98$\% of it's maximum value, is also given. Mean and standard deviation are computed over 5 trials.}
\label{tab:ar_eos}
\setlength{\tabcolsep}{3pt}
\begin{tabular}{c|cccc}
\toprule
$t_\mathrm{EOS}$ & $\hat{I}(\mathbf{X};\mathbf{H}^{140})$ & $\hat{I}(\mathbf{X};\mathbf{H}^{150})$ & $\Delta \hat{I}(\mathbf{X};\mathbf{H}^t)$  & Convergence epoch \\ 
\midrule
$0$ & $3.5 \pm 0.5$ & $3.5 \pm 0.5$ & $-0.06 \pm 0.01$ & $97 \pm 49$\\
$1$ & $5.2 \pm 0.7$ & $3.4 \pm 0.4$ & $-1.81 \pm 0.65$ & $53 \pm 14$\\
$10$ & $5.3 \pm 0.5$ & $2.5 \pm 0.4$ & $-2.87 \pm 0.40$& $59 \pm 22$
\end{tabular}
\end{table}

\paragraph{Alternating informative and uninformative phases}
In Section~\ref{sec:hypothesis}, we hypothesized that the compresssion constraint propagates to timesteps at which the LSTM is in a similar activation state. To investigate this claim, we modify the synthetic AR data-generation process by making randomly selected intervals uninformative: observations within these intervals are replaced with white noise. Although the locations of uninformative intervals vary randomly across training samples, the final 15 timesteps of each sequence are always either informative (\txtmode{informative end}) or uninformative (\txtmode{uninformative end}).

At evaluation time, we alternate between informative and uninformative phases every 25 timesteps. We consider both training conditions, with and without an EOS signal, using $t_\mathrm{EOS} \in \{0,25\}$ and sequence length $T=200$, yielding four experimental setups. For each setup, we measure the change in $\hat{I}(\mathbf{X};\mathbf{H}^t)$ and $\hat{I}(\mathbf{H}^t;Y)$ over each 25-timestep phase. That is, we compute $\Delta \hat{I}(\mathbf{X};\mathbf{H})$ and $\Delta \hat{I}(\mathbf{H};Y)$ from the onset to the offset of a phase. We exclude the first two phases and the final phase from this analysis.

Fig.~\ref{fig:alternation} compares the resulting information dynamics for networks trained with an \txtmode{informative end} and an \txtmode{uninformative end}. It also shows scatter plots and the aggregate mean and standard deviation of the information changes. Two effects are particularly clear. First, networks trained with an \txtmode{informative end} compress during informative phases, whereas networks trained with an \txtmode{uninformative end} exhibit the opposite behavior. Second, adding an EOS token substantially reduces compression during informative phases for networks trained with an \txtmode{informative end}. These results therefore strongly support the conjecture that constraint propagation to stochastically similar periods is responsible for prolonged compression.

\begin{figure}[t]
    \centering
\raisebox{-0.5mm}{%
    \includegraphics[width=0.60\linewidth]{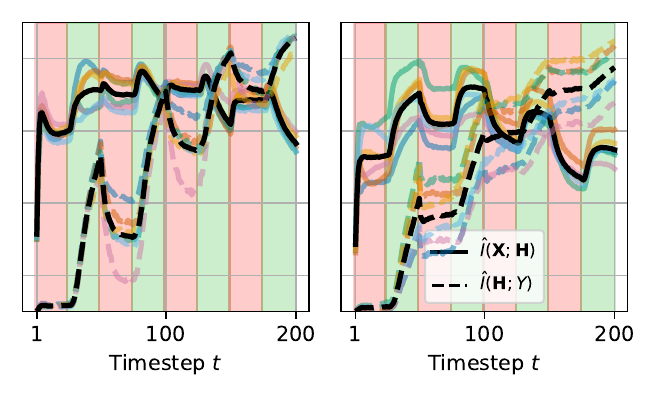}%
}
\includegraphics[width=0.38\linewidth]{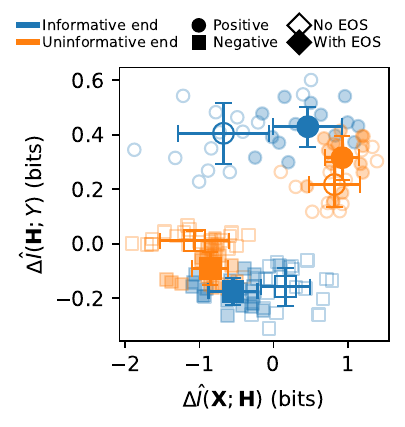}
    \caption{Information content behavior of a network trained and tested on AR dataset with alternating information. \textbf{Left} Embedding entropy $\hat{I}(\mathbf{X};\mathbf{H}^t)$ and target information $\hat{I}(\mathbf{H}^t; Y)$ over timesteps $t$ when training on \txtmode{informative end} and no EOS signal. Shown are the first 5 trials, with their average displayed in black. \textbf{Middle:} Same, but trained on \txtmode{uninformative end}. \textbf{Right:} Change in information content between the beginning and end of each phase, excluding the first two and final phases of each trial. Error bars indicate the mean and standard deviation for each color–symbol combination.}
    \label{fig:alternation}
\end{figure}

\subsection{Effect on Expressivity}
Finally, we examine how compression interacts with network capacity and generalization. To this end, we train LSTMs with one and two recurrent layers, with 10 neurons per layer, both without and with an EOS signal ($t_{\mathrm{EOS}} \in \{0,1\}$), on a synthetic AR dataset with $T=30$ we refer to as \dataset{AR short} and three real-world datasets. For statistical analysis, each configuration is trained 30 times with different weight initializations. Since unfavorable initializations can cause training to terminate prematurely due to numerical issues, we discard the three trials with the lowest maximum validation accuracy across all epochs, leaving 27 trials for each configuration.

Table~\ref{tab:eos_accuracy_K27} summarizes test set accuracies across the 27 trials. To investigate the relationship between embedding compression and generalization, Table~\ref{tab:eos_avg_pvalue_K27} reports Pearson correlations between test accuracy and time-averaged embedding entropy $1/T\sum_{t=1}^T \hat{I}(\mathbf{X}, \mathbf{H}^t)$ across 54 trials, pooling trials with and without the EOS signal. Time-averaged embedding entropy allows us to capture the cumulative effect of compression, as our hypothesis predicts that compression across many timesteps constrains the network's expressiveness.

Across datasets, higher average embedding entropy is generally associated with better generalization, although not all correlations are statistically significant. Interestingly, the effect of the EOS signal does not always align with this trend. On \dataset{Tiselac}, networks perform better without the EOS signal, whereas on \dataset{ElectricDevices}, they perform better with it. Yet, for the latter dataset, we find no evidence of a positive association between average embedding entropy and test accuracy, at least for networks with a single recurrent layer. These results suggest that the EOS signal is not a reliable mechanism for implicitly reducing compression pressure (see~\cite[Table~\ref{tab:eos_xmiavg_pvalue_K27}]{Kopp_LSTM_arXiv}). Note that these datasets are also those on which the networks perform poorly. The conflicting results may therefore arise because the networks are already severely underparameterized.

An alternative explanation for the improved generalization with EOS on \dataset{ElectricDevices}, despite the absence of a corresponding effect on compression, is that EOS may effectively increase the network's usable capacity. The recurrent embeddings $\mathbf{h}^t$ serve a dual purpose: providing context for subsequent timesteps and, at the final timestep, providing the penultimate layer representation for classification. The EOS signal may allow the network to decouple these roles, enabling it to use its recurrent embeddings differently for information propagation and the final classification task. This mechanism may also contribute to the improved generalization observed on \dataset{MNIST stroke} and \dataset{AR short}, suggesting that the benefits of EOS cannot be attributed solely to reduced overall compression. This does not, however, diminish the broader evidence for a link between embedding entropy and generalization across datasets and trials, as shown in Table~\ref{tab:eos_avg_pvalue_K27}.

\begin{table}[t]
\centering
\caption{Generalization performance of LSTM networks across four experimental setups for each dataset, combining the presence or absence of an EOS signal with either one or two LSTM layers. Shown are the mean and standard deviation of the accuracy over the best-performing 90\% of trials (27 of 30).}\label{tab:eos_accuracy_K27}
\begin{tabular}{lccc}
\toprule
Dataset & EOS & 1 Layer & 2 Layers \\
\midrule
\multirow{2}{*}{\dataset{MNIST stroke}} & $\times$ & $90.69 \pm 0.58$ & $93.99 \pm 0.42$ \\
 & $\checkmark$ & $91.12 \pm 0.64$ & $94.07 \pm 0.30$ \\
\midrule
\multirow{2}{*}{\dataset{Tiselac}} & $\times$ & $82.27 \pm 0.86$ & $82.11 \pm 0.87$ \\
 & $\checkmark$ & $81.11 \pm 1.29$ & $82.11 \pm 0.89$ \\
\midrule
\multirow{2}{*}{\dataset{ElectricDevices}} & $\times$ & $60.83 \pm 2.87$ & $65.38 \pm 2.12$ \\
 & $\checkmark$ & $61.43 \pm 3.15$ & $66.99 \pm 2.45$ \\
\midrule
\multirow{2}{*}{\dataset{AR short}} & $\times$ & $96.66 \pm 0.34$ & $97.30 \pm 0.36$ \\
 & $\checkmark$ & $97.09 \pm 0.25$ & $97.47 \pm 0.27$ \\
\bottomrule
\end{tabular}
\end{table}

\begin{table}[t]
\centering
\caption{Pearson correlations between test accuracy and time-averaged embedding entropy $1/T\sum_{t=1}^T \hat{I}(\mathbf{X}, \mathbf{H}^t)$, computed across all 54 of 60 trials, irrespective of whether an EOS signal was present. $p$-values $<0.05$ are shown in bold.}
\label{tab:eos_avg_pvalue_K27}
\begin{tabular}{lcccc}
\toprule
Dataset & \multicolumn{2}{c}{1 Layer} & \multicolumn{2}{c}{2 Layers} \\
\cmidrule(lr){2-3} \cmidrule(lr){4-5}
 & $r$ & $p$-value & $r$ & $p$-value \\
\midrule
\dataset{MNIST stroke} & 0.44 & $\mathbf{7.8 \times 10^{-4}}$ & 0.15 & $2.6 \times 10^{-1}$ \\
\midrule
\dataset{Tiselac} & 0.45 & $\mathbf{6.9 \times 10^{-4}}$ & 0.45 & $\mathbf{5.1 \times 10^{-4}}$ \\
\midrule
\dataset{ElectricDevices} & 0.03 & $8.3 \times 10^{-1}$ & 0.31 & $\mathbf{2.3 \times 10^{-2}}$ \\
\midrule
\dataset{AR short} & 0.39 & $\mathbf{3.8 \times 10^{-3}}$ & 0.07 & $5.9 \times 10^{-1}$ \\
\bottomrule
\end{tabular}
\end{table}

\section{Conclusion}
In this paper, we identify previously unreported compression dynamics across timesteps in discretized LSTMs and connect them with class separability and the interplay of minimizing the cross-entropy loss of a softmax classifier over a finite embedding space. We further show that prolonged compression, consistent with our hypothesis that (explicit or implicit) learning objectives propagate to timesteps with similar activation patterns, can impair generalization when network capacity is already constrained. Because our networks are underparameterized, these findings do not contradict prior work linking compression to improved generalization. Instead, they may be consistent with the view that compression has a beneficial regularizing effect when sufficient model capacity is available.

\section*{Acknowledgment}
ChatGPT (free web version) and ChatGPT 5.6 Terra (AcademicAI) were used to assist with language editing, improve the clarity and conciseness of the manuscript, develop data-visualization scripts, support literature research and assist in developing and refining theory. All AI-assisted outputs were reviewed and verified by the authors.

\bibliographystyle{IEEEtran}
\bibliography{references}

\vspace{12pt}

\clearpage

\appendices

\section{Penultimate embedding $\mathbf{H}^T$ collapse}
We first explain the straight-through estimator (STE)~\cite{bengio2013estimating}, a technique for applying gradient-based optimization to networks with non-differentiable components. We then present theoretical results on the effects of training a classifier on the geometry of quantized embeddings in the penultimate layer. 

\subsection{Straight-through estimator}
Gradient-based optimization can fail when gradients are undefined or uninformative. Networks with discrete components, such as activations based on the $\operatorname{sign}(\cdot)$ function, exhibit both issues: the gradient is zero almost everywhere and undefined at the discontinuity. However, Bengio et al.~\cite{bengio2013estimating} observed that when writing the gradient w.r.t. the weights as a chain of partial derivatives --- which, in effect, is exactly what is computed during backpropagation --- the only problematic derivatives are those of discrete components, such as $\partial \operatorname{sign}(\tilde{\mathbf{h}})/\partial \tilde{\mathbf{h}}$. They therefore proposed to replace these derivatives by surrogates, a heuristic technique they coined \emph{straight-through estimator}.

We now formalize this idea in the setting considered in this paper. Let us write a NN for classification as a concatenation of two functions $\hat{f}^o(\hat{f}^e(\mathbf{x})) = y$, where we term $\hat{f}^e{:}\ \mathcal{X} \to \mathcal{H} \subset \mathbb{R}^d$ the \emph{encoder} that produces bounded penultimate layer pre-embeddings $\tilde{\mathbf{h}}$, $\hat{f}^o{:}\ \mathcal{H} \to \Delta^{C-1}$ is an output layer that maps embeddings to class probabilities, $\mathcal{X}$ is the input space and $\mathcal{H}$ is a bounded set.\footnote{In our setting, $\mathcal{X} = \bigcup_{T\in\mathbb{N}} \mathbb{R}^{T\times F}$ and $\tilde{\mathbf{h}}$ is the final timestep embedding $\tilde{\mathbf{h}}^T$. We adopt a more general formulation in this section, since the analysis also applies to penultimate embeddings in feed-forward architectures.} Let us further denote the per-sample loss defined on $\tilde{\mathbf{h}}$ for a fixed output layer by $\ell_y{:}\  \mathcal{H} \to \mathbb{R}_+$, such that 
\begin{equation}
    \mathcal{L}(\{\mathbf{x}_n, y_n\}_{n=1}^N) = \frac{1}{N} \sum_{n=1}^N \ell_{y_n}(\hat{f}^e(\mathbf{x}_n)) = \frac{1}{N} \sum_{n=1}^N \ell_{y_n}(\tilde{\mathbf{h}}_n).
\end{equation}
Note that the input space of the output layer $\hat{f}^o(\cdot)$ and per-sample loss $\ell_y(\cdot)$ is $\mathcal{H} \subset \mathbb{R}^d$ and is therefore continuous. We now examine how gradient descent updates affect the pre-embeddings $\tilde{\mathbf{h}}$. Given a sample $\mathbf{x}_n$, backpropagation first computes the gradient of the loss with respect to its pre-embedding, $\partial \ell_y(\tilde{\mathbf{h}}_n)/\partial \tilde{\mathbf{h}}_n$, before propagating this gradient through the encoder $\hat{f}^e(\cdot)$ to update its parameters. In this section, however, we are not interested in the encoder itself, but in the direction in which gradient-based optimization seeks to move $\tilde{\mathbf{h}}_n$. We therefore consider updates directly in pre-embedding space, which can be interpreted as assuming an arbitrarily powerful encoder $\hat{f}^e(\cdot)$ which we can instruct to realize any mapping. The update is
\begin{equation}
\label{eq:backprop}
\Delta \tilde{\mathbf{h}}_n
= -\eta \left.
\frac{\partial \ell_y(\tilde{\mathbf{h}})}
{\partial \tilde{\mathbf{h}}}
\right|_{\tilde{\mathbf{h}}=\tilde{\mathbf{h}}_n}  \coloneqq -\eta \nabla \ell_{y_n}(\tilde{\mathbf{h}}_n).
\end{equation}
In the remainder of this section, we will drop the sample index $n$, since we are interested in the update of any $\mathbf{h} \in \mathcal{H} \subset \mathbb{R}^d$ given $y$. We now insert a quantizer $\mathcal{Q}{:}\ \mathbb{R}^d \to \{-1, 1\}^d \subset \mathbb{R}^d$ between encoder and output layer, yielding the function composition $\hat{f}^o \circ \mathcal{Q} \circ \hat{f}^e$.\footnote{Note that while the quantizer restricts the set of possible inputs to the output layer, the ambient space is still $\mathbb{R}^d$. In other words, the output layer $\hat{f}^o$ is still defined on $\mathbb{R}^d$.} Specifically, we define
\begin{equation}
\mathcal{Q}(\tilde{\mathbf{h}})_i \coloneqq \begin{cases}
-1, & \tilde{h}_i < 0,\\
1, & \tilde{h}_i \geq 0,
\end{cases}
\qquad i=1,\ldots,d,
\end{equation}
and write $\operatorname{sign}(\cdot)$ as shorthand for this component-wise operation. The quantizer renders Equation~\eqref{eq:backprop} useless, since the derivative of the composition $\ell_y \circ \mathcal{Q}(\tilde{\mathbf{h}})$ w.r.t. $\tilde{\mathbf{h}}$ is undefined at boundaries and zero elsewhere. To fix this, a component-wise surrogate derivative
\begin{equation}
\frac{\partial \mathcal{Q}(\tilde{\mathbf{h}})_i}{\partial \tilde{h}_i}
\overset{\mathrm{STE}}{:=} \rho(\tilde{h}_i),
\qquad \rho(\tilde{h}_i) > 0,
\end{equation}
also known as STE, is introduced. The update with the quantizer can now be written as
\begin{equation}
\label{eq:backprop_ste}
\Delta \tilde{\mathbf{h}} = -\eta \boldsymbol\rho(\tilde{\mathbf{h}}) \odot \nabla \ell_y(\operatorname{sign}(\tilde{\mathbf{h}})).
\end{equation}
When having many quantization levels, a typical choice for the surrogate is  $\rho(\cdot)=1$, which can be interpreted as treating the quantizer as the identity function when computing the gradient and hence known as \emph{identity-STE}. However, since we just have two quantization levels, we will use the derivative of $\operatorname{tanh}$ as $\rho$. Fig.~\ref{fig:ste} illustrates the updates described in Equations~\eqref{eq:backprop} and~\eqref{eq:backprop_ste}.

\begin{figure}[t]
    \centering
    \begin{subfigure}[b]{0.48\linewidth}
        \centering
        \includegraphics[width=\linewidth]{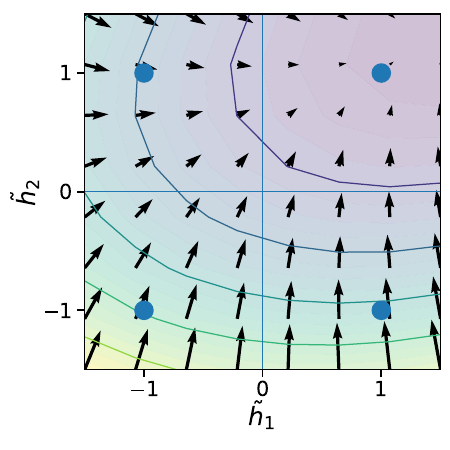}
        \caption{}
        \label{fig:a}
    \end{subfigure}
    \hfill
    \begin{subfigure}[b]{0.48\linewidth}
        \centering
        \includegraphics[width=\linewidth]{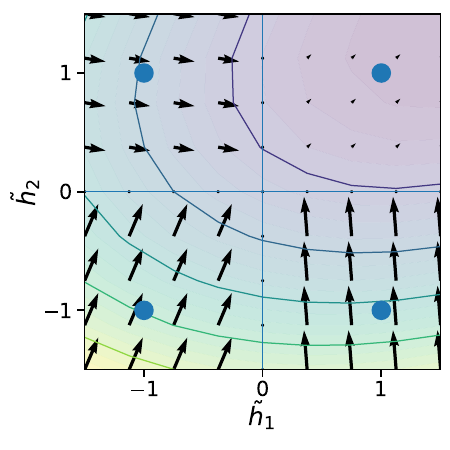}
        \caption{}
        \label{fig:b}
    \end{subfigure}

    \begin{subfigure}[b]{0.48\linewidth}
        \centering
        \includegraphics[width=\linewidth]{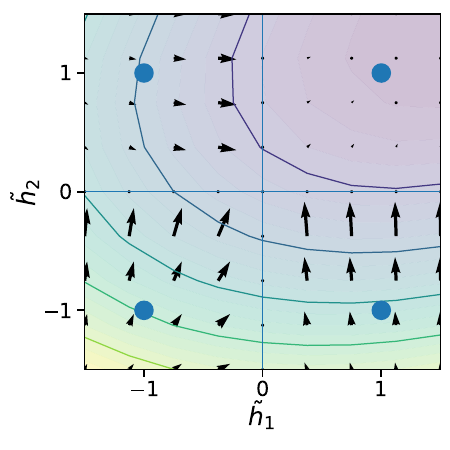}
        \caption{}
        \label{fig:c}
    \end{subfigure}
    \hfill
    \begin{subfigure}[b]{0.48\linewidth}
        \centering
        \includegraphics[width=\linewidth]{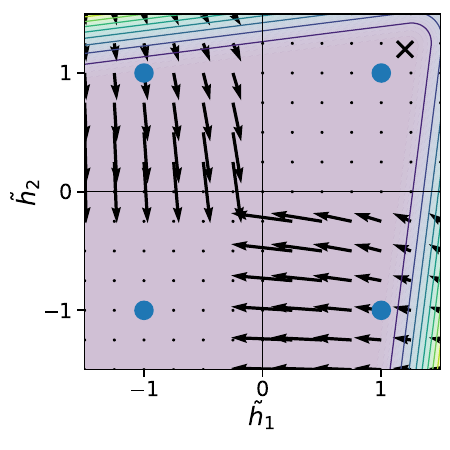}
        \caption{}
        \label{fig:d}
    \end{subfigure}

    \caption{Illustration of the straight-through estimator. The elliptical contour plot shows the level sets of the loss, and the arrows display the negative gradient field. \textbf{(a)} shows the updates when no quantizer is used, \ie Equation~\eqref{eq:backprop}. \textbf{(b)} displays the STE updates given by Equation~\ref{eq:backprop_ste}, using the derivative of the identity function as the surrogate $\rho(\cdot)$. Note that the update is fully determined by the gradient at the quantized values, namely the vertices of the square. Even though the minimum lies in the first quadrant, updates in the second quadrant may flip $h_2 = \operatorname{sign}(\tilde{h}_2)$, while updates in the fourth quadrant may flip $h_1 = \operatorname{sign}(\tilde{h}_1)$. In either case, such flips can increase distance to the loss-minimizing quadrant. Thus, STE updates are not guaranteed to eventually reach the loss-minimizing vertex $\mathbf{h}=(1,1)$. \textbf{(c)} shows the STE updates when using the derivative of $\operatorname{tanh}$ as the surrogate $\rho(\cdot)$. Note that vectors near the quantization boundaries are pulled towards them, as the derivative of $\operatorname{tanh}$ is largest near $0$. The vector field can be transformed into that of (b) by applying a strictly monotonic transformation to each coordinate. \textbf{(d)} shows an example where even though the loss is convex and STE moves us locally closer to the loss-minimizing orthant, we are not guaranteed to get closer with discrete steps.}
    \label{fig:ste}
\end{figure}

\subsection{Theoretical results}
\label{sec:theory}
\begin{definition}[Vertex loss difference]
Let $\mathcal{W}$ denote the 
network parameters. For every class $y$ and every pair of distinct vertices $\mathbf v,\mathbf v'\in\{-1,+1\}^d$, we define
\begin{equation}
D_{y,\mathbf v,\mathbf v'}(\mathcal{W}) \coloneqq \ell_{y,\mathcal{W}}(\mathbf v) - \ell_{y,\mathcal{W}}(\mathbf v').
\end{equation}
\end{definition}

\begin{assumption}[Non-degeneracy of vertex losses]
\label{ass:generic-ties}
For every class $y$ and every pair of distinct vertices $\mathbf v,\mathbf v'$, the function $D_{y,\mathbf v,\mathbf v'}$ is real analytic and not identically zero. 
\end{assumption}

\begin{lemma}[Uniqueness of loss-minimizing vertex $\mathbf v^y$]
\label{lem:unique-codeword}
Under Assumption~\ref{ass:generic-ties}, for every class $y$, the minimizer
$$
\mathbf v^y
\in
\argmin_{\mathbf v\in\{-1,+1\}^d}
\ell_{y,\boldsymbol\theta}(\mathbf v)
$$
is unique for almost all network parameters $\mathcal{W}$.
\end{lemma}

\begin{proof}
Fix a class $y$ and two distinct vertices
$$
\mathbf v,\mathbf v'\in\{-1,+1\}^d.
$$
A tie between these vertices occurs if and only if
$$
D_{y,\mathbf v,\mathbf v'}(\boldsymbol\theta)=0.
$$
By Assumption~\ref{ass:generic-ties},
$D_{y,\mathbf v,\mathbf v'}$ is a nonzero real-analytic function.
Its zero set therefore has Lebesgue measure zero. 

There are finitely many pairs of distinct vertices in
$\{-1,+1\}^d$. The proof is completed by the fact that a finite union of null sets is a null set.
%
\end{proof}

\begin{definition}[Loss-minimizing vertex, orthants and quantization boundary]
Let
$$
\mathcal V:=\{-1,+1\}^d
$$
and 
$$
\mathbf v^y
\in
\arg\min_{\mathbf v\in \mathcal{V}}
\ell_{y,\boldsymbol\theta}(\mathbf v)
$$
be a loss-minimizing vertex.

For each $\mathbf v\in\mathcal V$, define the open orthant
$$
\mathcal O_{\mathbf v}
\coloneqq
\left\{
\tilde{\mathbf h}\in\mathbb R^d:
\operatorname{sign}(\tilde{\mathbf h})=\mathbf v
\right\}
=
\left\{
\tilde{\mathbf h}\in\mathbb R^d:
\mathbf v\odot\tilde{\mathbf h}\succ\mathbf 0
\right\}.
$$
The quantization boundary is
$$
\mathcal B
\coloneqq
\bigcup_{j=1}^d
\left\{
\tilde{\mathbf h}\in\mathbb R^d:
\tilde h_j=0
\right\}.
$$
\end{definition}

\begin{proposition}[Local STE flow towards loss-minimizing orthant]
\label{lem:global-ste-progress}
Let $\ell_y{:}\ \mathbb R^d\to\mathbb R$ be differentiable and convex, and
let $\mathbf v^y\in\{-1,+1\}^d$ be its unique minimizer over
$\{-1,+1\}^d$. Suppose that the surrogate gradient is applied
componentwise and satisfies
$$
0<\epsilon\leq\rho_j(s)\leq C<\infty,
\qquad
j\in\{1,\ldots,d\},\ s\in\mathbb R.
$$
Then the surrogate-STE flow
$$
\dot{\tilde{\mathbf h}}
=
-\boldsymbol\rho(\tilde{\mathbf h})
\odot
\nabla\ell_y\left(\operatorname{sign}(\tilde{\mathbf h})\right)
$$
locally moves towards the loss-minimizing orthant
$\mathcal O_{\mathbf v^y}$ at every
$\tilde{\mathbf h}\notin\mathcal B$ satisfying
$\operatorname{sign}(\tilde{\mathbf h})\neq\mathbf v^y$.
\end{proposition}

\begin{proof}
Define the coordinate transformation
$$
\tilde z_j
=
\gamma_j(\tilde h_j)
\coloneqq
\int_0^{\tilde h_j}\frac{ds}{\rho_j(s)}.
$$
The assumptions on $\rho_j$ imply that $\gamma_j$ is a strictly
increasing bijection of $\mathbb R$ and that
$$
\operatorname{sign}(\tilde z_j)
=
\operatorname{sign}(\tilde h_j).
$$
By the chain rule,
$$
\dot{\tilde z}_j
=
\frac{1}{\rho_j(\tilde h_j)}\dot{\tilde h}_j
=
-\left[
\nabla\ell_y\left(\operatorname{sign}(\tilde{\mathbf z})\right)
\right]_j.
$$
Thus, in $\tilde{\mathbf z}$ coordinates, the flow is identity-STE. Although the transformation need not preserve the numerical locations
$\pm1$ in pre-embedding space, this is immaterial: the binary vertices
$\mathbf v\in\{-1,+1\}^d$ are codewords in quantized-output space, and
$\operatorname{sign}(\tilde{\mathbf z})
=\operatorname{sign}(\tilde{\mathbf h})$. Hence the transformation
preserves both orthant labels and the loss value associated with each
orthant.

Let $\mathbf v=\operatorname{sign}(\tilde{\mathbf z})$ denote the vertex
associated with the orthant containing $\tilde{\mathbf z}$. We define
$$
\mathcal I(\tilde{\mathbf z})
=
\left\{
i\in\{1,\ldots,d\}:
\operatorname{sign}(\tilde z_i)\neq v_i^y
\right\}.
$$
Thus, $\mathcal I(\tilde{\mathbf z})$ is the set of coordinates in which
the current orthant disagrees with the loss-minimizing orthant
$\mathcal O_{\mathbf v^y}$. The $L_1$ distance from
$\tilde{\mathbf z}$ to $\mathcal O_{\mathbf v^y}$ is therefore
$$
d_y(\tilde{\mathbf z})
\coloneqq
\sum_{i\in\mathcal I(\tilde{\mathbf z})}
-v_i^y\tilde z_i.
$$
Indeed, for every $i\in\mathcal I(\tilde{\mathbf z})$, the signs of
$\tilde z_i$ and $v_i^y$ disagree, so that
$-v_i^y\tilde z_i=|\tilde z_i|$ is precisely the distance that the
$i$\textsuperscript{th} coordinate must travel to reach the corresponding quantization
boundary.

For convenience, for any index set $\mathcal I$, write
$$
\langle\mathbf a,\mathbf b\rangle_{\mathcal I}
\coloneqq
\sum_{i\in\mathcal I}a_i b_i.
$$
This notation allows us to rewrite the preceding sum as
$$
d_y(\tilde{\mathbf z})
=
\left\langle-\tilde{\mathbf z},\mathbf v^y\right\rangle_
{\mathcal I(\tilde{\mathbf z})}.
$$

For $\mathbf v\neq\mathbf v^y$, uniqueness and convexity imply
$$
\nabla\ell_y(\mathbf v)^\top(\mathbf v^y-\mathbf v)<0,
$$
or, equivalently,
$$
\left\langle\nabla\ell_y(\mathbf v),\mathbf v^y\right\rangle_
{\mathcal I(\tilde{\mathbf z})}<0.
$$
Therefore, within every non-optimal orthant,
$$
\frac{d}{dt}d_y(\tilde{\mathbf z}(t))
=
\left\langle\nabla\ell_y(\mathbf v),\mathbf v^y\right\rangle_
{\mathcal I(\tilde{\mathbf z})}
<0.
$$
\end{proof}

Therefore, the STE flow locally moves towards the loss-minimizing orthant $\mathcal{O}_{\mathbf{v}^y}$, except on the quantization boundaries $\mathcal{B}$; crossing such a boundary changes the index set $\mathcal{I}(\cdot)$ and thus the local STE dynamics (Fig.~\ref{fig:d}). 

Note that Proposition~\ref{lem:global-ste-progress} makes no statement about trajectories within $\mathcal{O}_{\mathbf{v}^y}$; in particular, our assumptions do not imply that the loss-minimizing vertex $\mathbf{v}^y$ is stable.
\begin{corollary}[STE dynamics of frozen softmax classifier]
\label{cor:frozen-softmax}
For a frozen softmax classifier with cross-entropy loss, surrogate-STE dynamics locally move every non-optimal penultimate layer pre-embedding $\tilde{\mathbf{h}}$ towards the unique loss-minimizing orthant, if $\tilde{\mathbf{h}} \notin \mathcal{B}$.
\end{corollary}

\begin{proof}
Cross-entropy applied to softmax is differentiable and convex in its input. The claim then follows from Lemma~\ref{lem:unique-codeword} and Proposition~\ref{lem:global-ste-progress}.
\end{proof}

\section{Datasets}\label{app:datasets}
\subsection{Preprocessing of \dataset{MNIST stroke}}
This dataset was created by algorithmically transforming images into sequences of pen strokes~\cite{dejong2016mniststrokes}. Most sequence elements specify the direction of pen movements, moving the pen by one pixel at a time. A special token allows larger movements without drawing, corresponding to repositioning the pen for a new stroke. Digits consist of one or multiple strokes, followed by an EOS token at the end of the sequence. We performed three preprocessing steps: (1) removing samples containing more than three strokes, (2) removing the EOS token, and (3) resampling the dataset to a common sequence length. Step (1) and (3) are explained below.

\paragraph{Outliers} The process used to generate pen-stroke sequences from images works well for most digits, but can fail catastrophically for some. This failure is accompanied by generating too many strokes (see for example the digits in Fig.~\ref{fig:mnist:b}, which consist of 5 strokes). We therefore remove all samples containing more than three strokes. Fig.~\ref{fig:mnist:a} shows the distribution of the number of strokes per sample. As can be seen, this heuristic removes only a small fraction of the dataset.

\begin{figure}[t]
    \centering
    \begin{subfigure}[b]{0.48\linewidth}
        \centering
        \includegraphics[width=\linewidth]{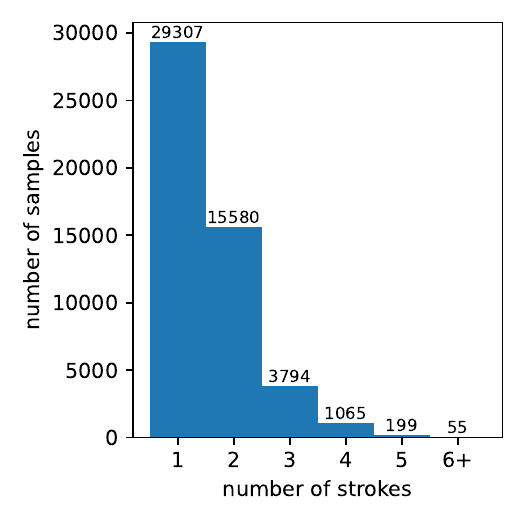}
        \caption{}
        \label{fig:mnist:a}
    \end{subfigure}
    \hfill
    \begin{subfigure}[b]{0.48\linewidth}
        \centering
        \includegraphics[width=\linewidth]{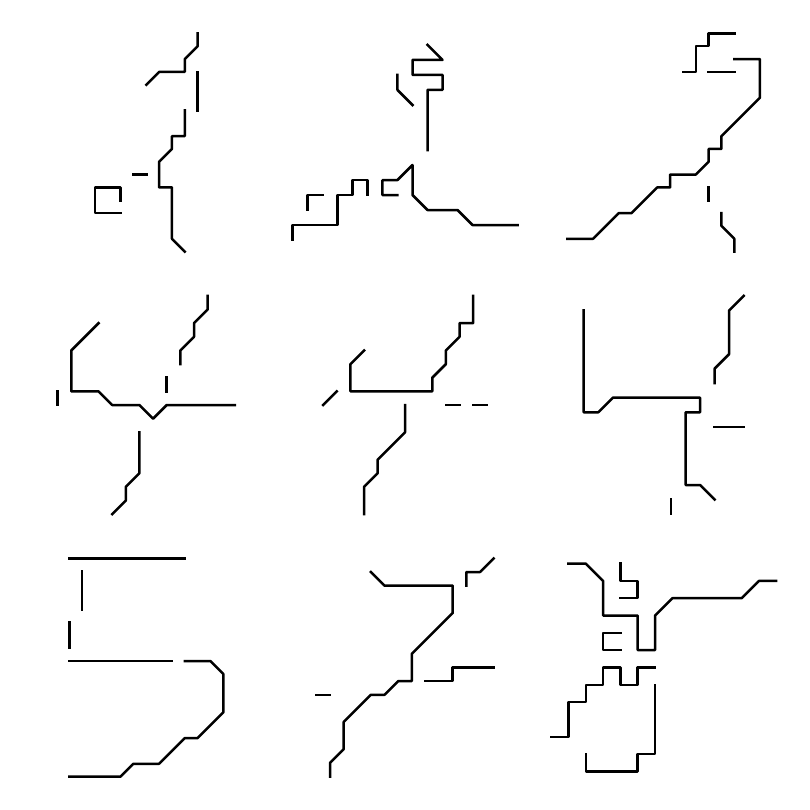}
        \caption{}
        \label{fig:mnist:b}
    \end{subfigure}

    \begin{subfigure}[b]{0.48\linewidth}
        \centering
        \includegraphics[width=\linewidth]{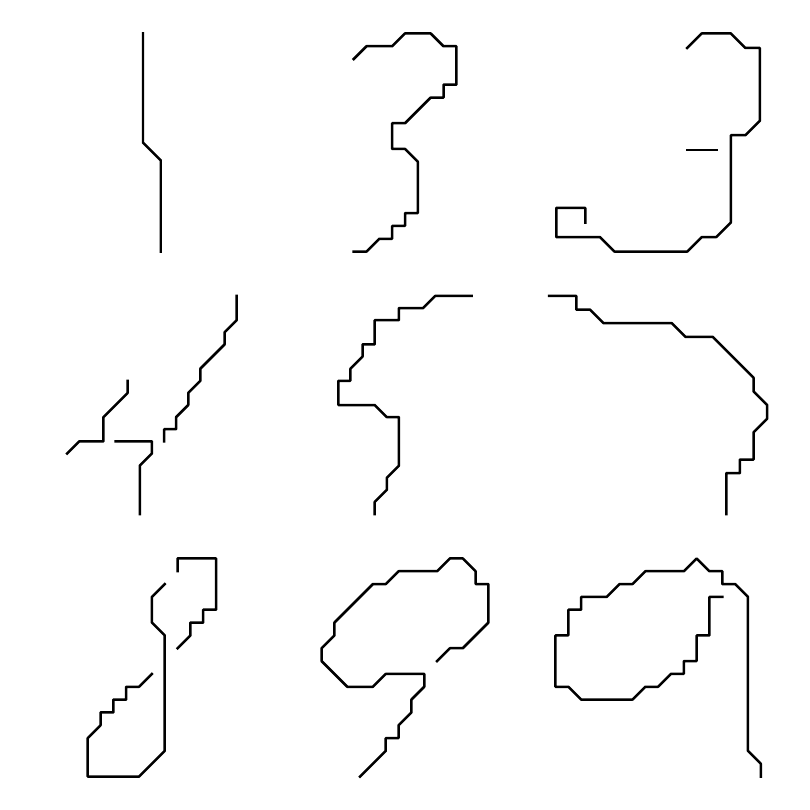}
        \caption{}
        \label{fig:mnist:c}
    \end{subfigure}
    \hfill
    \begin{subfigure}[b]{0.48\linewidth}
        \centering
        \includegraphics[width=\linewidth]{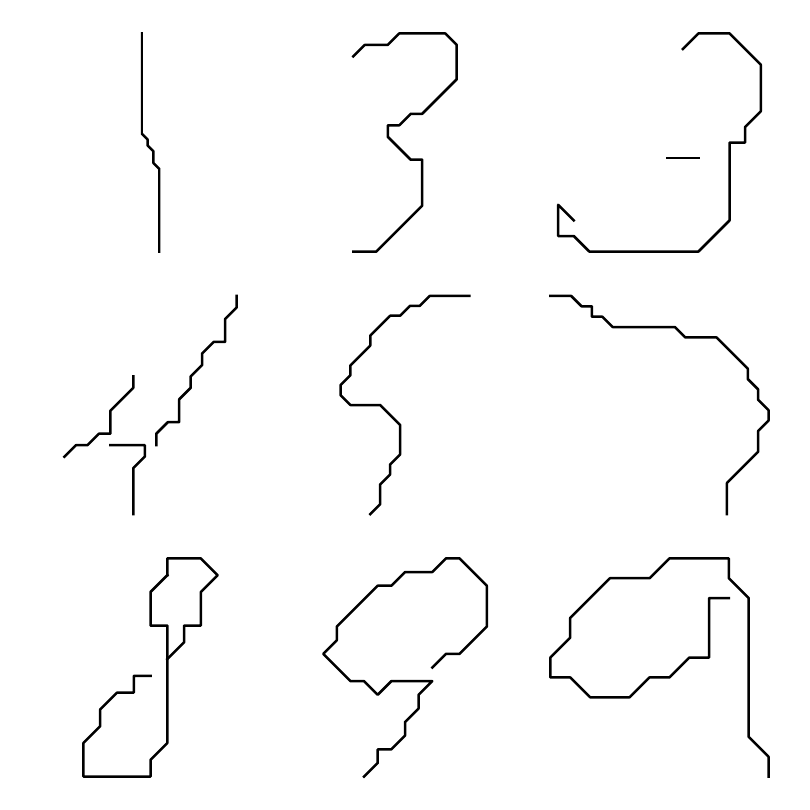}
        \caption{}
        \label{fig:mnist:d}
    \end{subfigure}
    \caption{Visualizations from the \dataset{MNIST stroke} dataset. The digit plots are generated by drawing the strokes as instructed by the sequence of pen moves. \textbf{(a)} shows the histogram of stroke counts in the training set. \textbf{(b)} shows randomly selected digits from sequences containing five strokes. \textbf{(c)} shows randomly selected digits from sequences containing three or fewer strokes, while \textbf{(d)} shows the corresponding digits generated from the resampled sequences.}
    \label{fig:mnist}
\end{figure}

\paragraph{Resampling}The original \dataset{MNIST stroke} dataset comprises sequences of varying lengths. While recurrent neural structures are specifically designed to handle variable-length sequences, our focus on studying embedding entropy and target information \emph{at concrete timesteps} runs into difficulty when sequences do not share a common length. To simplify matters, we opted to resample the sequences to a common length $T=38$, which is the median sequence length of the training set. The preprocessing procedure consists of the following steps:

\begin{enumerate}
\item The individual strokes and stroke-repositioning movements are extracted, and the rescaling ratio $r_s$ required to obtain the desired sequence length is determined. Stroke-repositioning movements are rescaled by this ratio but are otherwise left unchanged.

\item A length budget is allocated to each stroke according to its original seqence length. The budgets are chosen such that, together with the repositioning tokens, the resulting sequence has the desired common length.

\item Each stroke is represented as a two-dimensional curve, \ie a sequence of points obtained by cumulative summation of the displacements.

\item The elements are rescaled by the factor $r_s$ and smoothed coordinate-wise using a Savitzky--Golay filter. This reduces artifacts for sequences that undergo substantial up-scaling.

\item The curve is then divided by \emph{target points} placed equidistantly w.r.t. its Euclidean arc length. To this end, the curve is linearly interpolated between its given points. The number of target points is one greater than the length budget assigned to the stroke. See Fig.~\ref{fig:preprocessing} for an illustration of this and subsequent steps.

\item The current position is initialized at the start of the stroke, \ie $(0, 0)$. Target points are then processed in order. For each target point, its displacement from the current discrete position is computed, rounded component-wise to the nearest integer, and clipped to $\{-1,0,1\}$. The resulting discrete movement is then used to update the current position. The resulting sequence of discrete movements constitutes the resampled stroke. Note that $(0,0)$ movements are possible.

\item Finally, the resampled strokes are merged with the corresponding stroke-repositioning movements to reconstruct the complete sequence.

\end{enumerate}

\begin{figure}[t]
    \centering
    \includegraphics[width=\linewidth]{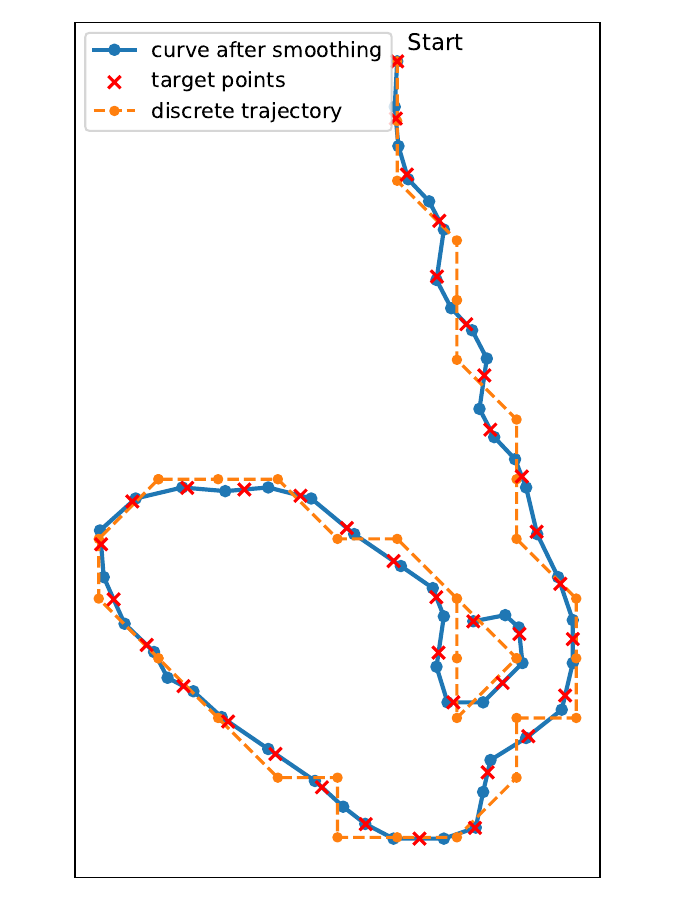}
    \caption{Illustration of the stroke resampling procedure. The blue curve shows the stroke after rescaling and smoothing. Target points are placed equidistantly with respect to the Euclidean arc length of the blue curve and are generally located between its elements by linear interpolation. The orange curve is the resulting discrete stroke, where each movement is chosen to approach the next target point as closely as possible subject to the discrete movement constraints. The original digit had a longer sequence length, and therefore the associated blue curve is downscaled by the factor $r_s$ before resampling, resulting in the shorter steps taken by blue compared to orange trajectory.}
    \label{fig:preprocessing}
\end{figure}

\subsection{Synthetic auto-regressive dataset generation}
\label{app:ar_dataset}
The aim of this synthetic dataset generation procedure is to provide a flexible framework for classification tasks. In particular, it allows the amount of target information in the input sequence to be controlled over time, enabling the \emph{informativeness} of individual timesteps to be varied. The main concept is the following: Sequence generation is driven by one of two class-specific autoregressive processes of order three (AR(3)) 
\begin{equation}
x^t = x[t] = \sum_{j=1}^3 \Phi_j x[t-j] + \epsilon[t],
\end{equation}
where $\{\Phi_j\}_j = \boldsymbol{\Phi}$ are the autoregressive coefficients, and $\epsilon^t \sim \mathcal{N}(0, \sigma_\epsilon^2)$ is the innovation process with variance $\sigma_\epsilon^2$. The classification task is to identify which of the two processes underlies a given sequence. The class-specific process is obscured either by interpolating the AR(3) sequence with a neutral stochastic process, or by reducing the differences between the class-specific AR(3) parameters. We describe the procedure in detail in the following.

We parameterize AR(3) processes using reflection coefficients $\mathbf{k}$ rather than conventional autoregressive coefficients $\boldsymbol{\Phi}$. This parameterization offers several advantages. Stationarity is ensured by constraining each reflection coefficient to $(-1,1)$, and the process variance has a simple closed-form expression in terms of $\mathbf{k}$. Moreover, the reflection coefficients are connected to prediction error of optimal linear regressors: The square of the $j$\textsuperscript{th} coefficient $k_j$ describes the fractional reduction in prediction error achieved by using $j$ lags instead of $j-1$. We choose the coefficients $\mathbf{k}_y = [0.7,\,0.7,\,0.5(-1)^y]^T$ with $y \in \{1, 2\}$. Thus, the class-specific processes only differ in the sign of the third reflection coefficient. The innovation variance $\sigma_\epsilon^2$ is set such that the processes have unit marginal variance: $\sigma_\epsilon^2 = \sum_{j=1}^3(1-k_j^2)$. 

To obtain sequences whose informativeness with respect to the target varies over time, we introduce a sequence $\lambda_\iota[t]\in[0,1]$, which controls the degree of target information at each timestep. A value of $\lambda_\iota[t]=0$ corresponds to no target information, whereas $\lambda_\iota[t]=1$ corresponds to maximally informative input, in the sense that $\lambda_\iota[t]=1$ means that the sequence elements are generated exclusively by the introduced AR(3) processes. We use two different mechanisms, which we call \txtmode{measurement noise} and \txtmode{information noise}. With \txtmode{measurement noise}, we generate a class-specific AR(3) sequence $\tilde{x}[t]$ and then mix it with white noise $\epsilon[t]\sim\mathcal{N}(0,1)$ via interpolation $x[t]=\sqrt{\lambda_\iota[t]}\tilde{x}[t]+\sqrt{1-\lambda_\iota[t]}\epsilon[t]$. This can be understood as a per-timestep linear interpolation which preserves variance. \txtmode{Information noise}, in contrast, modifies the reflection coefficients $\mathbf{k}_y=[0.7,\,0.7,\, \lambda_\iota[t]0.5(-1)^y]$ while generating the sequence. The innovation variance is adjusted accordingly, $\sigma_\epsilon^2[t] = \left[1-(\lambda_\iota[t]0.5)^2\right](1-0.7^2)^2$, to compensate for changes in variance. However, in contrast to \txtmode{measurement noise},  this guarantees unit marginal variance only when $\lambda_\iota[t]$ is constant over time. In all cases, the AR(3) processes were burnt in for 50 timesteps, during which $\lambda_\iota[t]$ was set to $1$. Note that the experiment reported in the main body of the paper uses \txtmode{measurement noise} with $\lambda_\iota(t) \in \{0, 1\}$.

\paragraph{Dataset setup for alternating informative phases}
To test how the dynamics of a trained network differ depending on whether the training set ends with an informative phase, we created two different types of training sets. The \txtmode{informative end} training set ends only on informative phases, and the \txtmode{uninformative end} training set only on uninformative phases. For each sample, a sequence $\lambda_\iota[t]$ is generated as follows. First, the number of toggle points is sampled uniformly from the integers $\{1,\ldots,n_{\max}\}$, where
$n_{\max} = \left\lfloor \frac{T}{\Delta_{\min}} \right\rfloor - 1$, $T$ is the sequence length and $\Delta_{\min}=15$ is the minimum interval length. The sequence is then divided into $n_t+1$ intervals, where $n_t$ is the sampled number of toggle points. Each interval is initialized to the minimum length $\Delta_{\min}$. The remaining timesteps are then randomly distributed among the intervals according to a multinomial distribution with equal probabilities. Within an interval, $\lambda_\iota[t]$ is constant, but it alternates between $0$ and $1$ between intervals, \ie at toggle points, with the initial value chosen such that the sequence ends at the prescribed value $\lambda_\iota[t]$, which is $1$ for \txtmode{informative end} and $0$ for \txtmode{uninformative end}. Figure~\ref{fig:lambdas} shows randomly drawn $\lambda_\iota[t]$ sequences for training with \txtmode{informative end}. 

Training LSTMs directly in the \txtmode{uninformative end} mode proved challenging, as the input never provides target information during the final 15 timesteps, making optimization difficult. To address this issue, we employ \emph{curriculum learning} to encourage the networks to learn extracting and retaining task-relevant information, regardless of the training mode. Specifically, we generate two training sets, $\mathcal{D}_{tr}^{(0)}$ and $\mathcal{D}_{tr}^{(1)}$, and train the networks on each for 50 epochs. Within each training set, all sequences share the same $\lambda_\iota[t]$. For $\mathcal{D}_{tr}^{(0)}$, we set $\lambda_\iota[t]=1$ throughout the sequence, allowing the networks to learn to extract task-relevant information. For $\mathcal{D}_{tr}^{(1)}$, we set $\lambda_\iota[t]=1$ for $t<195$ and $\lambda_\iota[t]=0$ during the final five timesteps. This is shown in Figure~\ref{fig:lambdas}. To evaluate the networks, we created test sets with one $\lambda_\iota[t]$ for all $\mathbf{x} \in \mathcal{D}_{te}$. This sequence starts with $\lambda_\iota[1]=0$ and toggles between $0$ and $1$ every 25 step (see Figure~\ref{fig:lambdas}).

\begin{figure}[t]
    \centering
    \includegraphics[width=\linewidth]{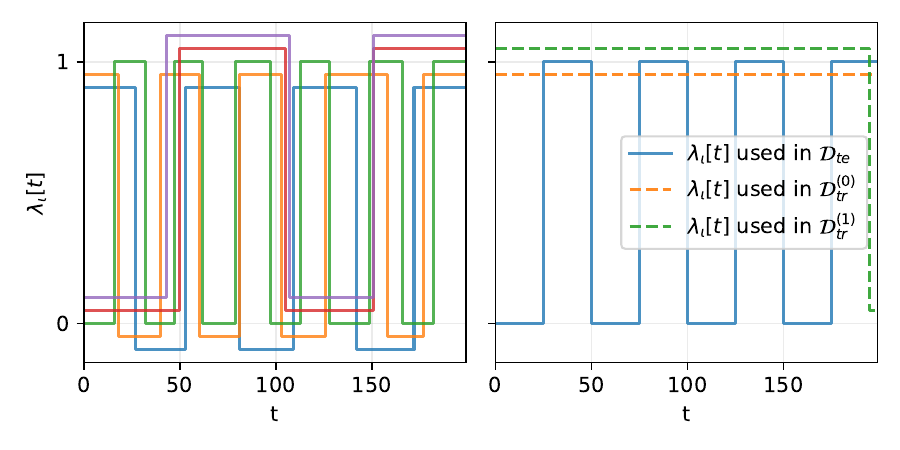}
    \caption{Sequences $\lambda_\iota[t]$ used to generate the AR datasets, controlling the informativeness of the input sequences $\mathbf{x}$. For visualization, a vertical offset was added to each sequence; in reality, $\lambda_\iota[t]\in{0,1}$. Left: Examples of $\lambda_\iota[t]$ used to generate training sequences in the \txtmode{informative end} mode. Right: The $\lambda_\iota[t]$ sequence used for all test sequences, together with the $\lambda_\iota[t]$ sequence used during the early epochs of curriculum learning.}
    \label{fig:lambdas}
\end{figure}

\section{Additional results}
In this appendix we provide additional plots and tables for two of the experiments described in the main text.  

\subsection{Alternating informative and uninformative phases}
In the main body, we report results only for \txtmode{measurement noise}. Here, we additionally investigate \txtmode{information noise} during uninformative phases. The data-generation process and experimental setup are described in Appendix~\ref{app:ar_dataset} and the main body, respectively. Each experimental setup was evaluated over ten trials. Note that $\lambda_\iota(t)$ varies between training and testing, whereas the EOS setting ($t_\mathrm{EOS} = 0$ or $1$) and noise type (\txtmode{measurement noise} or \txtmode{information noise}) does not.

Figs.~\ref{fig:alternating_measurement} and~\ref{fig:alternating_information} show the embedding entropy $\hat{I}(\mathbf{X};\mathbf{H}^t)$ and target information $\hat{I}(\mathbf{H}^t;Y)$ over timestep $t$ for \txtmode{measurement noise} and \txtmode{information noise}, respectively. For both noise conditions, we compute the change in information content between the onset and offset of each 25-timestep phase, excluding the first two and final phases of each trial. The result is shown in Fig.~\ref{fig:alternating_combined}, which compares change in $\hat{I}(\mathbf{X};\mathbf{H})$ and $\hat{I}(\mathbf{H};Y)$ across informative and uninformative phases, with and without an EOS signal. 

\begin{figure}[t]
    \centering
    \includegraphics[width=\linewidth]{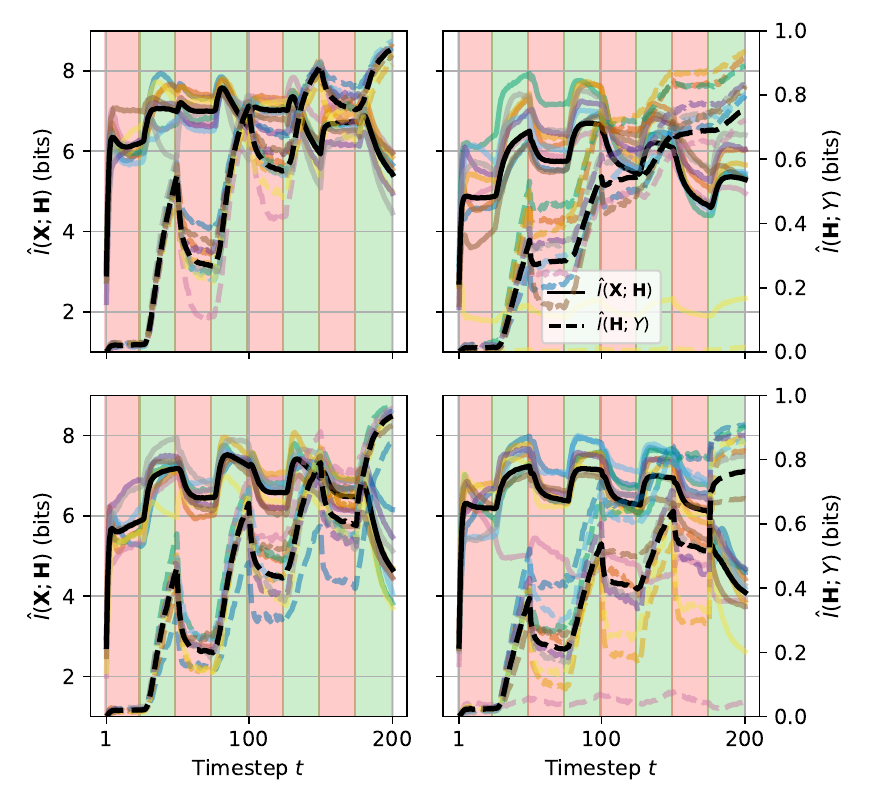}
    \caption{Information content behavior of networks trained and tested on AR dataset with alternating \txtmode{measurement noise}. Shown is embedding entropy $\hat{I}(\mathbf{X};\mathbf{H}^t)$ and target information $\hat{I}(\mathbf{H}^t; Y)$ over timesteps $t$ for all ten trials. The average embedding entropy and target information, aggregated over trials for each timestep $t$, is shown in black. Time intervals with $\lambda_\iota(t) = 0$ in the test sequences are shown in red, while those with $\lambda_\iota(t) = 1$ are shown in green. \textbf{Top: } No EOS information present during training and testing ($t_{\mathrm{EOS}}=0$). \textbf{Bottom: } EOS information is present during training and testing with $t_{\mathrm{EOS}}$.  \textbf{Left: } Networks are trained on \txtmode{informative end}. \textbf{Right: } Networks are trained on \txtmode{uninformative end}.}
    \label{fig:alternating_measurement}
\end{figure}

\begin{figure}[t]
    \centering
    \includegraphics[width=\linewidth]{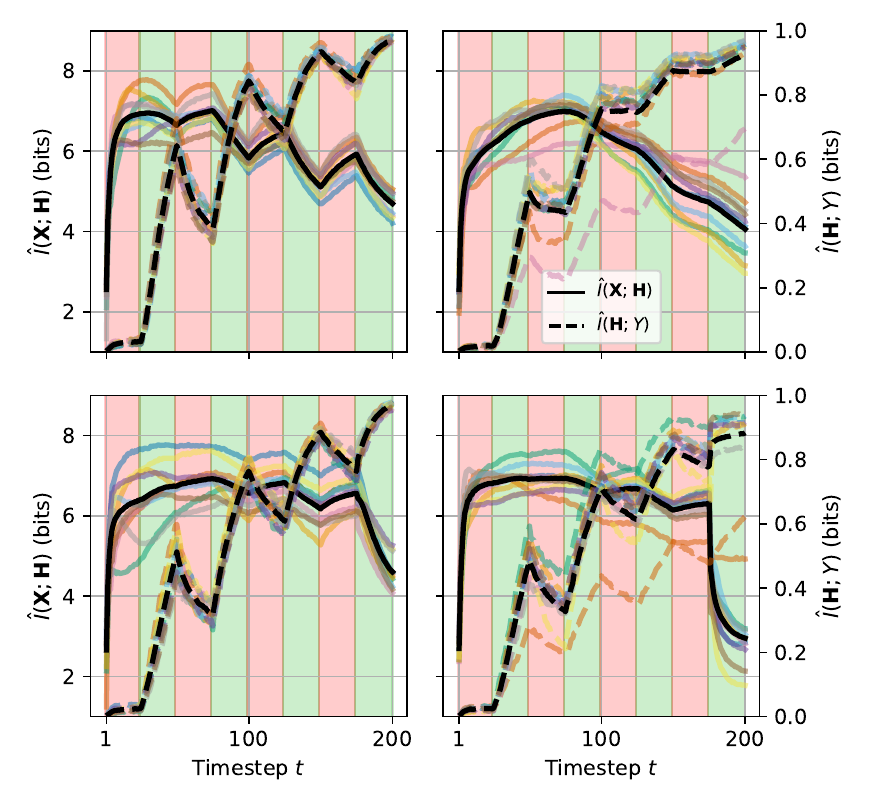}
    \caption{Information content behavior of networks trained and tested on AR dataset with alternating \txtmode{information noise}. Shown is embedding entropy $\hat{I}(\mathbf{X};\mathbf{H}^t)$ and target information $\hat{I}(\mathbf{H}^t; Y)$ over timesteps $t$ for all ten trials. The average embedding entropy and target information, aggregated over trials for each timestep $t$, is shown in black. Time intervals with $\lambda_\iota(t) = 0$ in the test sequences are shown in red, while those with $\lambda_\iota(t) = 1$ are shown in green. \textbf{Top: } No EOS information present during training and testing ($t_{\mathrm{EOS}}=0$). \textbf{Bottom: } EOS information is present during training and testing with $t_{\mathrm{EOS}}$.  \textbf{Left: } Networks are trained on \txtmode{informative end}. \textbf{Right: } Networks are trained on \txtmode{uninformative end}.}
    \label{fig:alternating_information}
\end{figure}

\begin{figure}[t]
    \centering
    \includegraphics[width=\linewidth]{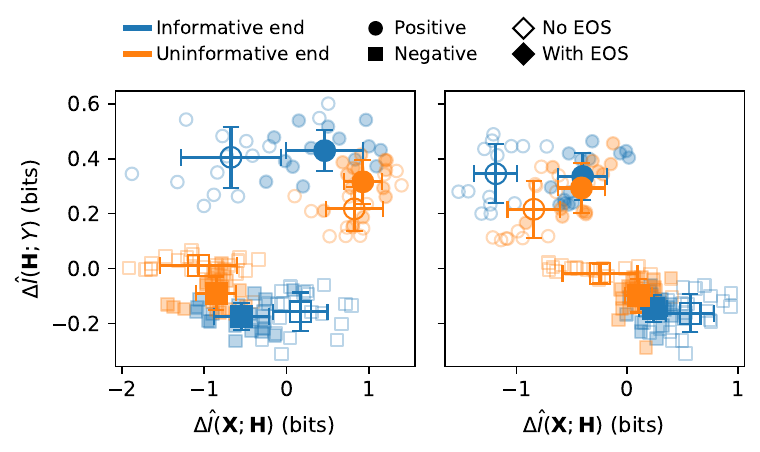}
    \caption{Effect of training on \txtmode{informative end} compared to \txtmode{uninformative end} on the information content dynamics, with and without an EOS token. Shown is the change in information content between the beginning and end of each phase, excluding the first two and final phases of each trial. Error bars indicate the mean and standard deviation for each color–symbol combination. \textbf{Left: } Uninformative phases during training and testing were generated using \txtmode{measurement noise}. \textbf{Right: } Same, except that \txtmode{information noise} was used for uninformative phases.}
    \label{fig:alternating_combined}
\end{figure}

For both \txtmode{measurement} and \txtmode{information noise}, networks exhibit stronger compression in the phase that occurs at the last timestep (see Fig.~\ref{fig:alternating_combined}). Additionally, the inclusion of the EOS signal ($t_\mathrm{EOS}=25)$ reduces the discrepancy between \txtmode{informative} and \txtmode{uninformative end} in both noise types. Despite these common trends, the results also differ depending on how uninformative phases are generated. In general, \txtmode{information noise} leads to compression during informative phases and expansion during uninformative phases. It is unclear what causes this inversion, but we hypothesize that it reflects a fundamental difference in how the two noise types remove task-relevant information. Since the two AR(3) processes differ only in the sign of the third reflection coefficient, the task can be interpreted as predicting the next observation from the first two lags and then determining whether the third lag is positively or negatively correlated with the resulting prediction residual. Setting the third reflection coefficient to zero (\txtmode{information noise}) removes this sign-dependent cue while preserving the underlying autoregressive structure. The resulting sequences therefore retain temporal structure similar to that of informative sequences, but no longer provide information about which process generated them. The LSTMs may consequently treat these intervals as sequences in which the relevant information is difficult to extract, rather than as irrelevant input. In contrast, replacing observations with i.i.d. noise (\txtmode{measurement noise}) disrupts the autoregressive structure and changes the sequence statistics more substantially. The network may therefore be able to recognize these intervals as uninformative and learn to discount them. If this hypothesized difference in network behavior--—either ignoring the input or attempting to extract information from it—--explains the observed discrepancies in compression and expansion, it would suggest a yet-unexplored relationship between embedding entropy and the computational strategies employed by the LSTMs. Investigating this could be an interesting direction for future work.

\subsection{Effect on Expressivity}
This section contains additional results on the experiments on expressivity. Table~\ref{tab:eos_xmiavg_K27} shows mean and standard deviation of the \emph{average} embedding entropy, where the average is computed across timesteps. Table~\ref{tab:eos_xmiavg_pvalue_K27} presents the corresponding Welch's t-test results for the null hypothesis that adding an EOS token does not affect the average embedding entropy. Table~\ref{tab:eos_pvalue_K27} reports the Welch's t-test results for the null hypothesis that adding an EOS token does not affect test accuracy.

\begin{table}[t]
\centering
\caption{Average embedding entropy, aggregated over timesteps, of LSTM networks across four experimental setups for each dataset, combining the presence or absence of an EOS signal with either one or two LSTM layers. Shown are the mean and standard deviation of the entropy in bits, over the best-performing 90\% of trials (27 of 30).}\label{tab:eos_xmiavg_K27}
\begin{tabular}{lccc}
\toprule
Dataset & EOS & 1 Layer & 2 Layers \\

\midrule
\multirow{2}{*}{\dataset{MNIST stroke}} & $\times$ & $6.00 \pm 0.24$ & $6.08 \pm 0.24$ \\
 & $\checkmark$ & $6.21 \pm 0.31$ & $6.39 \pm 0.20$ \\
\midrule
\multirow{2}{*}{\dataset{Tiselac}} & $\times$ & $6.45 \pm 0.27$ & $5.81 \pm 0.29$ \\
 & $\checkmark$ & $6.35 \pm 0.32$ & $5.97 \pm 0.41$ \\
\midrule
\multirow{2}{*}{\dataset{ElectricDevices}} & $\times$ & $5.15 \pm 0.44$ & $5.22 \pm 0.27$ \\
 & $\checkmark$ & $4.93 \pm 0.48$ & $5.38 \pm 0.35$ \\
 \midrule
\multirow{2}{*}{\dataset{AR short}} & $\times$ & $6.43 \pm 0.45$ & $4.96 \pm 0.36$ \\
 & $\checkmark$ & $6.88 \pm 0.36$ & $6.01 \pm 0.44$ \\
\bottomrule
\end{tabular}
\end{table}

\begin{table}[t]
\centering
\caption{Statistical results of Welch's $t$-test assessing the effect of EOS signal presence on the average embedding entropy reported in Table~\ref{tab:eos_xmiavg_K27}. A positive $t$-statistic means the presence of EOS leads to higher average embedding entropy.}
\label{tab:eos_xmiavg_pvalue_K27}
\setlength{\tabcolsep}{3pt}
\begin{tabular}{lcccc}
\toprule
Dataset & \multicolumn{2}{c}{1 Layer} & \multicolumn{2}{c}{2 Layers} \\
\cmidrule(lr){2-3} \cmidrule(lr){4-5}
 & $t$-statistic & $p$-value & $t$-statistic & $p$-value \\
\midrule
\dataset{MNIST stroke} & 2.7 & $\mathbf{1.1 \times 10^{-2}}$ & 5.2 & $\mathbf{3.5 \times 10^{-6}}$ \\
\midrule
\dataset{Tiselac} & -1.2 & $2.2 \times 10^{-1}$ & 1.7 & $1.1 \times 10^{-1}$ \\
\midrule
\dataset{ElectricDevices} & -1.8 & $8.0 \times 10^{-2}$ & 1.9 & $6.5 \times 10^{-2}$ \\
\midrule
\dataset{AR short} & 4.0 & $\mathbf{1.8 \times 10^{-4}}$ & 9.7 & $\mathbf{5.3 \times 10^{-13}}$ \\
\bottomrule
\end{tabular}
\end{table}

\begin{table}[t]
\centering
\caption{Statistical results of Welch's $t$-test assessing the effect of EOS signal presence on the accuracies reported in Table~\ref{tab:eos_accuracy_K27}. $p$-values $<0.05$ are shown in bold. A positive $t$-statistic means the presence of EOS leads to higher accuracy.}
\label{tab:eos_pvalue_K27}
\setlength{\tabcolsep}{3pt}
\begin{tabular}{lcccc}
\toprule
Dataset & \multicolumn{2}{c}{1 Layer} & \multicolumn{2}{c}{2 Layers} \\
\cmidrule(lr){2-3} \cmidrule(lr){4-5}
 & $t$-statistic & $p$-value & $t$-statistic & $p$-value \\
\midrule
\dataset{MNIST stroke} & 2.5 & $\mathbf{1.4 \times 10^{-2}}$ & 0.7 & $4.7 \times 10^{-1}$ \\
\midrule
\dataset{Tiselac} & -3.8 & $\mathbf{4.3 \times 10^{-4}}$ & 0.0 & $9.9 \times 10^{-1}$ \\
\midrule
\dataset{ElectricDevices} & 0.7 & $4.8 \times 10^{-1}$ & 2.5 & $\mathbf{1.4 \times 10^{-2}}$ \\
\midrule
\dataset{AR short} & 5.2 & $\mathbf{4.1 \times 10^{-6}}$ & 2.0 & $5.1 \times 10^{-2}$ \\
\bottomrule
\end{tabular}
\end{table}

Figs.~\ref{fig:ds_trials_a0} and~\ref{fig:ds_trials_a1} show the information content over timesteps $t$ across all 30 trials for LSTMs with one and two recurrent layers, respectively. The corresponding accuracies on the test set are displayed in Figs.~\ref{fig:hists_a0} and~\ref{fig:hists_a1}. As evident, some networks fail to learn the task adequately despite being trained for many epochs. Since these networks can be identified by their low training and validation accuracies, this supports our decision to exclude the 10\% worst-performing trials from the statistical analysis in the main text. The failures primarily affects networks trained on \dataset{ElectricDevices}, but we apply the exclusion to all networks for consistency.

Across real-world datasets, we observe less compression than on the synthetic \dataset{AR short}. Our hypothesis provides three explanations: (i) the embeddings may be harder to separate, consistent with the lower accuracies observed on real-world datasets, resulting in weaker pressure for final-timestep compression; (ii) the inputs are non-stationary, which is particularly true for \dataset{MNIST stroke} and \dataset{ElectricDevices}, limiting the propagation of information across timesteps; and (iii) the real-world datasets may require higher network complexity, leaving the networks more  underparameterized. However, if the embeddings are insufficiently separable to begin with, neither stationarity nor model capacity is expected to affect compression, as there is little incentive for final-timestep compression in the first place.  For example, we observe increased compression upon adding another layer only for \dataset{MNIST stroke} and \dataset{AR short} (compare Fig.~\ref{fig:ds_trials_a0} and Fig.~\ref{fig:ds_trials_a1}). We believe that the absence of this effect on the other datasets reflects insufficient embedding separability, and thus a lack of compression pressure for final-timestep embeddings.

For \dataset{ElectricDevices}, we even observe that the EOS signal increases embedding entropy at the final timestep. As discussed in the main text, the EOS signal may help the network decouple the dual roles of its embeddings: providing recurrent information that controls the gates at the next timestep and providing the penultimate-layer representation used by the output layer. Although such an explanation does not specifically predict an increase in embedding entropy, neither does it preclude one. Notably, on this dataset increases in target information $\hat{I}(\mathbf{H}^t; Y)$ are almost always accompanied by increases in embedding entropy $\hat{I}(\mathbf{X};\mathbf{H}^t)$. We do not analyze this effect further.

\begin{figure}[t]
    \centering
    \includegraphics[width=\linewidth]{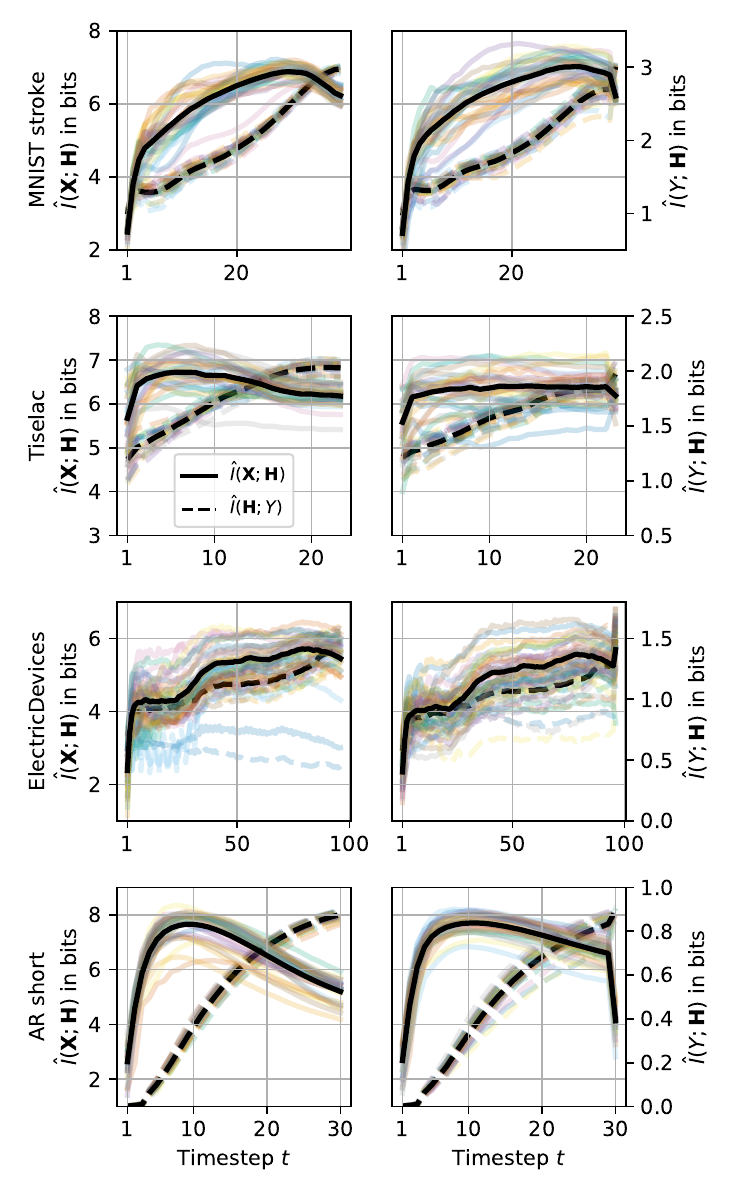}
    \caption{Embedding entropy $\hat{I}(\mathbf{X};\mathbf{H}^t)$ and target information $\hat{I}(\mathbf{H}^t; Y)$ over timesteps $t$ for different datasets, showing all 30 independently trained LSTMs per dataset, each with a single recurrent layer. The black lines show the respective averages across trials at each timestep. The dataset is indicated by the label on the left. \textbf{Left:} Without EOS signal, \textbf{Right:} With EOS signal ($t_\mathrm{EOS}=1$).}
    \label{fig:ds_trials_a0}
\end{figure}

\begin{figure}[t]
    \centering
    \includegraphics[width=\linewidth]{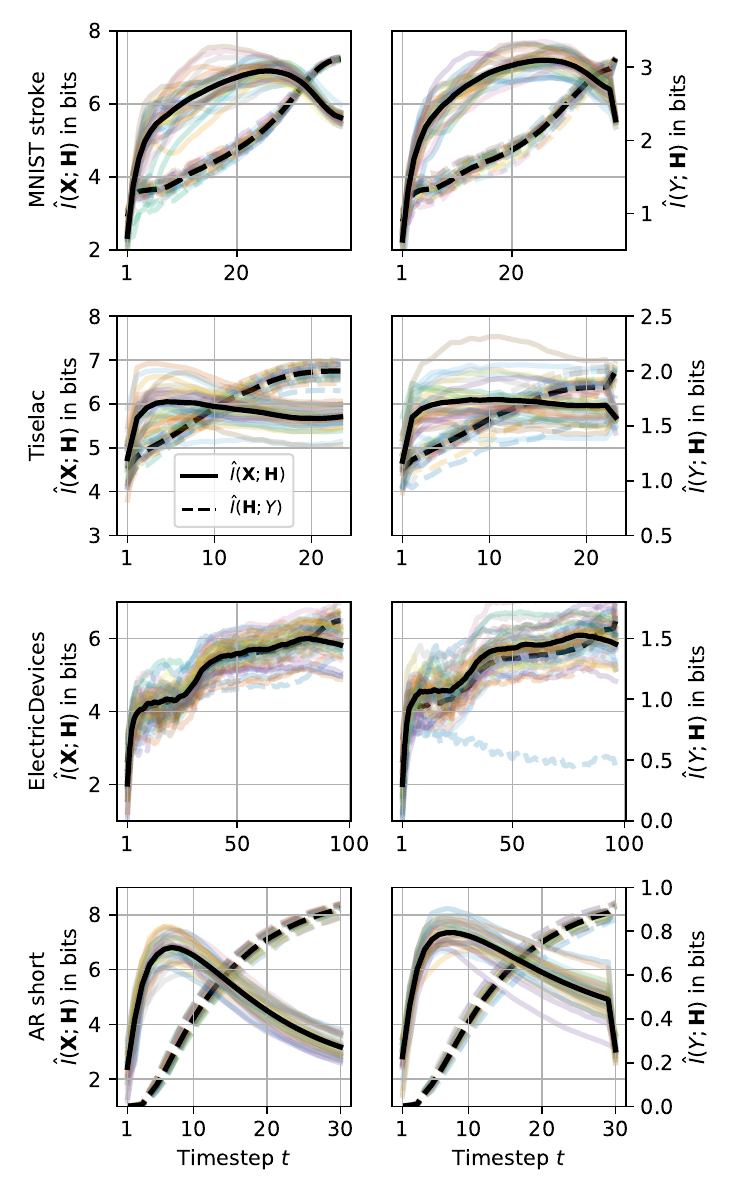}
    \caption{Embedding entropy $\hat{I}(\mathbf{X};\mathbf{H}^t)$ and target information $\hat{I}(\mathbf{H}^t; Y)$ over timesteps $t$ for different datasets, showing all 30 independently trained LSTMs per dataset, each with two recurrent layers. The black lines show the respective averages across trials at each timestep. The dataset is indicated by the label on the left. \textbf{Left:} Without EOS signal, \textbf{Right:} With EOS signal ($t_\mathrm{EOS}=1$).}
    \label{fig:ds_trials_a1}
\end{figure}

\begin{figure}[t]
    \centering
    \includegraphics[width=\linewidth]{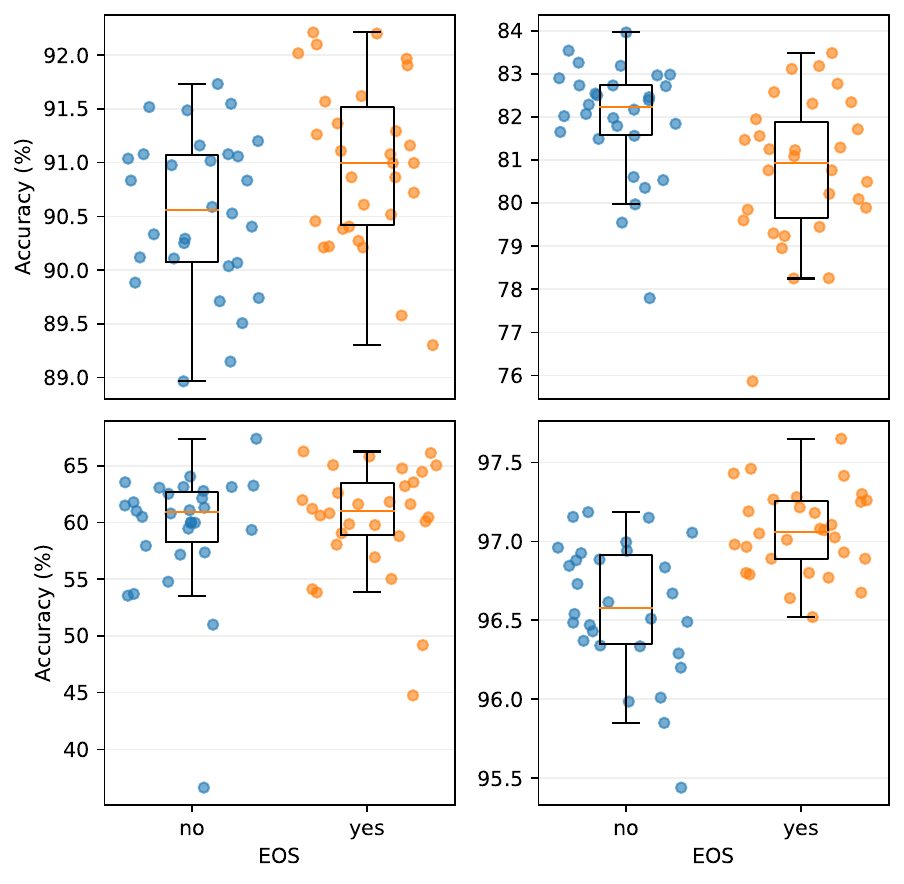}
    \caption{Accuracies of all 30 independently trained LSTMs per dataset with a single recurrent layer, using post-training early stopping. \textbf{Top left: } \dataset{MNIST stroke}. \textbf{Top right: } \dataset{Tiselac}. \textbf{Bottom left: } \dataset{ElectricDevices}. \textbf{Bottom right: } \dataset{AR short}.}
    \label{fig:hists_a0}
\end{figure}

\begin{figure}[t]
    \centering
    \includegraphics[width=\linewidth]{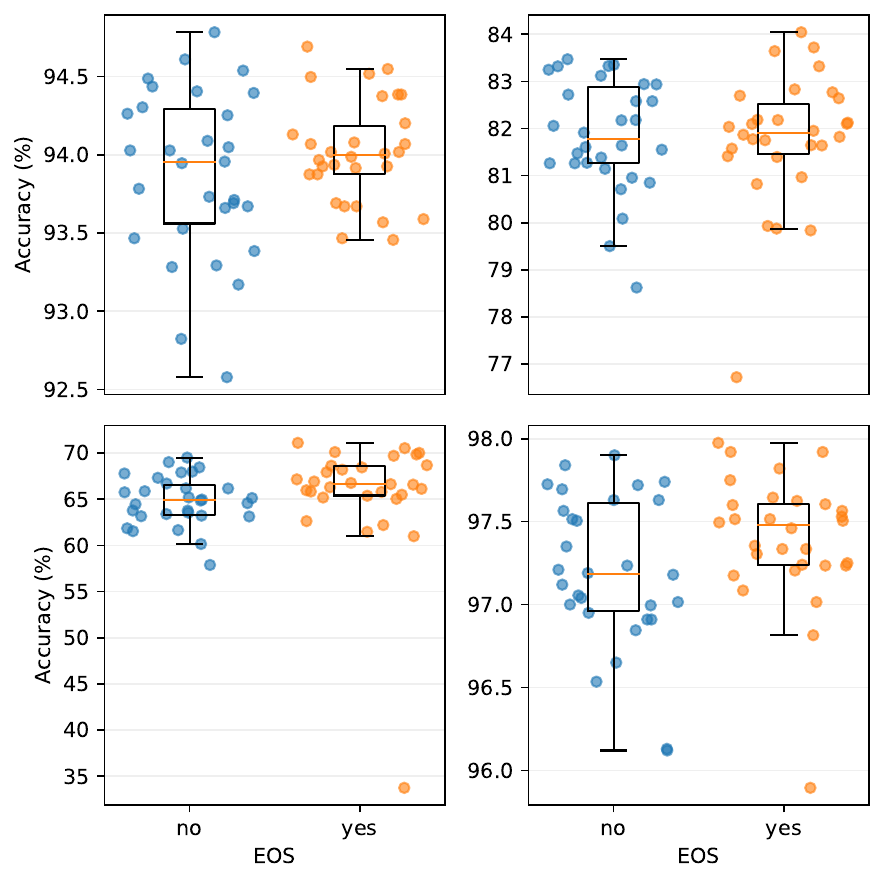}
    \caption{Accuracies of all 30 independently trained LSTMs per dataset with two recurrent layers, using post-training early stopping. \textbf{Top left: } \dataset{MNIST stroke}. \textbf{Top right: } \dataset{Tiselac}. \textbf{Bottom left: } \dataset{ElectricDevices}. \textbf{Bottom right: } \dataset{AR short}.}
    \label{fig:hists_a1}
\end{figure}

\end{document}